%% file: neurips_2026.tex
\documentclass{article}

\PassOptionsToPackage{numbers,sort&compress}{natbib}

\usepackage[preprint]{neurips_2026}

\usepackage[utf8]{inputenc}
\usepackage[T1]{fontenc}

\usepackage[colorlinks=true,citecolor=blue,linkcolor=blue,urlcolor=blue]{hyperref}
\usepackage{url}
\usepackage{xcolor}
\usepackage{nicefrac}
\usepackage{microtype}

\usepackage{amsmath,amssymb,amsfonts,amsthm,mathtools,bm}

\usepackage{booktabs}
\usepackage{multirow}
\usepackage{makecell}
\usepackage{graphicx}
\usepackage{subcaption}
\usepackage{float}

\usepackage{algorithm}
\usepackage[noend]{algpseudocode}

\newtheorem{theorem}{Theorem}[section]
\newtheorem{corollary}{Corollary}[theorem]
\newtheorem{lemma}[theorem]{Lemma}
\newtheorem{proposition}[theorem]{Proposition}
\newtheorem{remark}[theorem]{Remark}

\theoremstyle{definition}
\newtheorem{definition}{Definition}[section]

\newtheoremstyle{myexample}
  {6pt}
  {6pt}
  {\normalfont}
  {}
  {\bfseries}
  {.}
  {0.5em}
  {}

\theoremstyle{myexample}
\newtheorem{example}{Example}[section]

\definecolor{nischalcolor}{RGB}{220,68,5}

\title{Learning with Conflicts of Interest}

\author{%
  Nischal Aryal\textsuperscript{1} \quad
  Arash Termehchy\textsuperscript{1} \quad
  Ali Vakilian\textsuperscript{2} \quad
  Marianne Winslett\textsuperscript{3}
  \\
  \textsuperscript{1}Oregon State University
  \quad
  \textsuperscript{2}Virginia Tech
  \quad
  \textsuperscript{3}University of Illinois
  \\
  \texttt{\{aryaln,termehca\}@oregonstate.edu}
  \quad
  \texttt{vakilian@vt.edu}
  \quad
  \texttt{winslett@illinois.edu}
}
\begin{document}

\maketitle

\begin{abstract}
Financial, social, and political factors often prevent the interests of the owners of ML systems and services and their users from being perfectly aligned. ML systems often produce biased information that can influence users to make decisions that are not in their best interest. Current solution approaches require ML systems to implement protocols to mitigate their biases. However, ML system owners usually do not have any incentive to implement these protocols and often argue that it limits their freedom of expression or business. We believe that a successful solution to this problem must recognize the conflict of interest between the ML systems and their users, and use this information to protect users against information that adversely influences their decisions while allowing users to safely benefit from these systems. To this end, we propose a game-theoretic framework that models the interaction between ML systems and users with conflicts of interest. 
We present scalable algorithms with theoretical guarantees that maximize the amount of desired information and actions and minimize the amount of biased and manipulative actions in interaction with ML systems.
\end{abstract}


\input{1-Introduction.tex}
\input{1.5-Related-Work}

\input{2-Framework.tex}

\input{3-Possibility-of-Influence}
\input{4-Optimal-Influence}

\input{5-Noise-In-Learning}

\input{5-Experiments}
\bibliographystyle{plainnat}
\bibliography{reference}
\newpage
\appendix
\input{Appendix}

\end{document}

%% file: 1-Introduction.tex
\section{Introduction}
\label{sec:introduction}
Popular ML systems and services are usually developed and operated by government or business entities. 
Financial, social, and political factors prevent the incentives of these entities and their users from being perfectly aligned \citep{yeh2025position}.
For example, companies are often said to rework their recommendation and ranking models to promote their own products and undermine those of their competitors \citep{Fussell2019,Mattioli2019,Nicas2019}.
News sources often provide biased recommendations that can convince users to make decisions based on the goals of the website owner \citep{summerfield2024advancedaisystemsimpact,Zraik2025}.
Social networks often promote extreme content or points of view, which is not necessarily in the best interest of their users or society \citep{10.1145/3533379,10.1145/3533380, summerfield2024advancedaisystemsimpact}. 
The access to a wealth of background knowledge 
using foundation models has given AI agents a substantial ability to influence user decisions \citep{summerfield2024advancedaisystemsimpact,yeh2025position,DBLP:journals/corr/abs-2404-15058, jones2024liesdamnedliesdistributional,bozdag2025persuadecanframeworkevaluating,10.1145/3586183.3606763}.

Current proposals to combat these undesirable actions and influences of ML systems require them to implement certain rules and restrictions \citep{yeh2025position,10.1145/3533380}. 
However, these restrictions usually conflict with the objectives of the entities behind these systems, so they do not have an incentive to implement them. 
In addition, some argue that these restrictions limit the freedom to consume, learn from, and distribute information or conduct business and may not be ethical or would have long-term adverse impacts on society.
Also, they often impose significant computational overhead \citep{yeh2025position}. 

In this paper, we propose a game‐theoretic framework for supervised learning in settings where the user and the learner have conflicts of interest.
In our framework, the user has private training data that, if truthfully conveyed to the ML system (learner), may lead the system to learn a biased model to the benefit of the system. Importantly, our framework also captures the case where the system first learns an unbiased model and then injects bias at inference time. In both cases, the model's predictions undermine the user's interest, e.g., by providing biased recommendations.

However, the owner of the ML system and the user do share some interests: the ML system should satisfy the user's objective well enough to keep them engaged.  If the system learns an overly biased model, the inferences using this model may be too biased and discourage the users from using the system, e.g., recommending a set of products that only remotely reflect the training data collected from the user. 
We can leverage this partial common interest and design algorithms that take the user's true data as input and compute the optimal training data to share with the learner to minimize the bias of the learned model.
But two can play this game!
The ML system may be aware that users know about its undesirable biases in learning, e.g., based on their previous experiences. Thus, it may try to interpret the input training data strategically to offset the user's modifications, e.g., by using its prior information about the user.

We propose novel algorithms that find the equilibria of this reasoning and detect the equilibria in which the user and the learner have enough common interest to make learning an accurate model possible.
We also provide efficient algorithms that find strategies for the user to modify the training data so that the resulting data reflects the true training data, to learn an accurate model, and convince the learner to deploy the least biased model possible in the setting.

%% file: 1.5-Related-Work.tex
\section{Related Work}
Our model is related to several active areas in strategic and robust machine learning, but it is distinct in both
\emph{what is manipulated} and \emph{how the interaction is modeled}. In our setting, the user privately observes a dataset
$D \sim \pi$ and sends a \emph{costless, unverifiable report} $M \in \mathcal{M}$ to influence the learner's choice of
hypothesis $h \in \mathcal{H}$. The learner then updates beliefs about $D$ and chooses a hypothesis that maximizes its
expected utility. Thus, the strategic object is not an individual feature vector or a corrupted training point, but the
\emph{reported dataset} itself, and the interaction is modeled as a Bayesian information-transmission game.

This differs from \emph{strategic classification}, where agents manipulate their own observable features at deployment time
in response to a fixed decision rule~\citep{hardt2016strategicclassification,bechavod2022information,cohen2024strategic,milli2019social,dong2018strategic}. There, the central question is how a classifier changes behavior when individuals
best-respond through costly feature manipulation. In contrast, in our model the user does not adapt features to a deployed
classifier; instead, the user strategically controls the \emph{information released before learning}, anticipating how the
learner will interpret the report and choose $h$.

It also differs from the literature on \emph{adversarial examples}, which studies worst-case test-time perturbations of an
input within a geometric budget such as an $\ell_p$ ball~\citep{szegedy2013intriguing,goodfellow2014explaining,madry2018towards}. That literature is primarily minimax and robustness-driven. By
contrast, our model is incentive-driven: the user chooses a report to maximize $U^r(D,h)$, the learner responds according
to $U^\ell(D,h)$, and equilibrium depends on beliefs and strategic signaling rather than norm-bounded perturbations.

Finally, our setting is distinct from \emph{data poisoning}, where an attacker modifies the training pipeline by injecting or
altering data to degrade the learned model~\citep{biggio2012poisoning,steinhardt2017certified,jagielski2018manipulating,mei2015using}. Poisoning treats the learner as training directly on corrupted samples. In our
model, the key friction is instead \emph{information asymmetry}: the learner knows only the prior $\pi$ and the received
report $M$, and must infer the hidden dataset through Bayes' rule. For this reason, our framework is closer in spirit to cheap talk and Bayesian persuasion with endogenous learning objectives~\citep{kamenica2011bayesianpersuasion,crawford1982cheaptalk} than to standard poisoning models. 

In summary, compared with these related areas, our framework focuses on \emph{strategic dataset reporting under conflicting
utilities}, where equilibrium behavior is governed by signaling, posterior beliefs, and the learner's downstream choice of
hypothesis.

%% file: 2-Framework.tex
\section{Model}
\label{sec:framework}
We study a strategic game between a user and a learner in the setting of supervised learning.

\paragraph{Private dataset.} Let $\mathcal X \subseteq \mathbb R^d$ be the feature space and $\mathcal Y$ the label space. A private dataset is $D=((\bm x_1,y_1),\ldots,(\bm x_n,y_n)) \in \mathcal D = (\mathcal X \times \mathcal Y)^n$, where the samples $(\bm x_i,y_i)$ may be interpreted as i.i.d. draws from an underlying data-generating environment. The user picks their private dataset from a prior
distribution $\pi$, which is known to the learner.

\begin{example}
\label{example:loan-scoring}
Consider a loan-scoring problem with feature space $\mathcal X=[0,1]^2$, where $\bm x_i=(x_{i1},x_{i2})$ records the normalized income and credit score of applicants. A simple private dataset is
$D=((\bm x_1,y_1),(\bm x_2,y_2),(\bm x_3,y_3))$
with $\bm x_1=(1,0)$, $\bm x_2=(0,1)$, $\bm x_3=(1,1)$, and
$\bm y=(0.2,0.2,0.8)^\top$ where $y_i$ represents the loan qualification score of the applicant with feature vector $\bm x_i$.
\end{example}

\paragraph{Dataset Report.}
For a private dataset $D$, the user submits a dataset report to the learner $M\in \mathcal{M}$, where $\mathcal M$ is an arbitrary report space. The dataset report may not be equivalent to $D$, e.g., to hide the user’s private dataset or to manipulate the learner. For instance, in Example~\ref{example:loan-scoring}, a report may keep labels fixed and alter feature values, e.g., change all credit and income features to $1$. A \textbf{user strategy} is a (stochastic) mapping from a private dataset to a dataset report $\sigma_r:\mathcal D \to \Delta(\mathcal M)$.

\underline{Reports are \emph{costless} and \emph{unverifiable}}: the learner observes $M$ but cannot directly verify whether it is a truthful description of the private dataset. 



\paragraph{Hypothesis.} Let $\mathcal H \subseteq \mathbb R^{\mathcal X}$ be a hypothesis class, where $\mathbb R^{\mathcal X} = \{h:\mathcal X \to \mathbb R\}$ -- for instance, linear functions $\mathcal H=\{h_{\bm w}:\mathcal X\to\mathbb R \mid h_{\bm w}(\bm x)=\bm x^\top \bm w,\ \bm w\in W\}$ for some feasible parameter set $\mathcal W\subseteq \mathbb R^d$. The learner's action is the choice of a hypothesis $h\in\mathcal H$,
or equivalently model parameters $\bm{w}\in \mathcal W$. A \textbf{learner strategy} is a mapping from the dataset report to a hypothesis, $\sigma_l:\mathcal M\to \Delta(\mathcal H)$. The learner uses $\sigma_l$ to interpret the private information behind the submitted report.


\paragraph{Utility Functions.} The utility functions of the user and the learner measure the desirability of the hypothesis for the user and the learner, respectively. The learner's and user's utility functions are denoted by $U^\ell(D,h)$ and $U^r(D,h)$, respectively. Both map $\mathcal D\times\mathcal H$ into $\mathbb R$. 
If there is a conflict of interest between the user and the learner, $U^r$ and $U^\ell$ take their maximum values for different values of $h$ for the same private dataset $D$. 


\begin{example}
\label{example:utility-conflict}
Continuing Example~\ref{example:loan-scoring}, let $\mathcal H=\{h:\bm x\mapsto \bm x^\top\bm w,\ \bm w\in\mathbb R^2\}$. Assume the learner's utility is the negative empirical squared loss:
$U^\ell(D,h)=-\hat L^\ell(D,\bm w)$,
where
$\hat L^\ell(D,\bm w)=\|\bm X\bm w-\bm y\|_2^2$.
Hence
$\bm w_\ell^*(D)\in\arg\max_{\bm w}U^\ell(D,\bm w)
=\arg\min_{\bm w}\hat L^\ell(D,\bm w)
=\left(\frac13,\frac13\right)$. Suppose the user prefers systematically more favorable loan scores, so that both income and credit score should receive larger weight parameters. 
Assume the user's utility has the same squared-loss form, but with a conflict bias vector
$\bm b=(0.1,0.1)^\top$ and 
$
U^r(D,h)= -\|\bm X(\bm w-\bm b)-\bm y\|_2^2.
$
Then,
$
\bm w_r^*(D)\in\arg\max_{\bm w}U^r(D,h)
=\left(\frac{13}{30},\frac{13}{30}\right)
$. Thus, for the same dataset $D$, the learner's and the user's utility functions are maximized for different ideal hypotheses. 
\end{example}


\paragraph{Common knowledge.}
We assume that the prior $\pi$ over $\mathcal D$, the report space $\mathcal M$,
and the utility functions $U^\ell$ and $U^r$ are common knowledge. This follows
the standard game-theoretic convention used in strategic classification, where
agents' costs/utilities and the learner's decision rule are model primitives
~\citep{hardt2016strategicclassification,cohen2024strategic,dong2018strategic}, as well as in adversarial learning and data poisoning, which explicitly specify the learner’s training rule and the adversary’s objective~\citep{biggio2012poisoning,steinhardt2017certified}.

\paragraph{Equilibria.} As explained in Section~\ref{sec:introduction}, the user and the learner use their available information to reason about the strategies of
the other party to maximize their utility function. We would like to identify whether this reasoning will converge to any eventual stable state in which the user and learner settle on a pair of
fixed strategies where neither can increase their utility by deviating from their strategies.   
Since the learner does not know the user's private dataset, we follow the approach used in the game theory literature and use the concept of {\it Bayesian equilibrium} to characterize the stable state of the user and learner reasoning \citep{kamenica2011bayesianpersuasion,crawford1982cheaptalk}. For ease of exposition, we present the equilibrium definition for the finite-support case, e.g., $\mathcal{D}$ is finite. In the continuous or infinite-support case, the corresponding sums are replaced by integrals. 

The \emph{belief of the learner} is a posterior distribution
$\Phi(\cdot \mid M)\in \Delta(\mathcal D)$, where for each report \(M\),
\(\Phi(D\mid M)\) is the posterior probability that the private dataset is \(D\).
Given a user strategy \(\sigma_r\), Bayes' rule requires $\Phi(D\mid M)
=
\frac{\sigma_r(M\mid D)\pi(D)}
{\sum_{D'\in\mathcal D}\sigma_r(M\mid D')\pi(D')}$
whenever the denominator is positive. For a fixed private dataset \(D\), the \textit{user's
expected utility} from sending report \(M\) is
\(\mathbb E_{h\sim\sigma_l(\cdot\mid M)}[U^r(D,h)]
=\sum_{h\in\mathcal H}U^r(D,h)\sigma_l(h\mid M)\). For a fixed report \(M\), the
\textit{learner's expected utility} from choosing hypothesis \(h\) is
\(\mathbb E_{D\sim\Phi(\cdot\mid M)}[U^\ell(D,h)]
=\sum_{D\in\mathcal D}U^\ell(D,h)\Phi(D\mid M)\).

\vspace{-1mm}
\begin{definition}
\label{definition:equilibrium}
An equilibrium for the learning game is a pair of user and learner strategies
\((\sigma_r(M\mid D),\sigma_l(h\mid M))\)
whenever \(\sigma_r(M\mid D)>0\),
\(M\in\arg\max_{M'}\mathbb E_{h\sim\sigma_l(\cdot\mid M')}[U^r(D,h)]\); and whenever
\(\sigma_l(h\mid M)>0\),
\(h\in\arg\max_{h'}\mathbb E_{D\sim\Phi(\cdot\mid M)}[U^\ell(D,h')]\).
\end{definition}
At equilibrium, the user sends reports that are optimal given how the learner interprets them, and the learner chooses hypotheses that are optimal given beliefs induced by those reports. 

Equilibria are stable states of mutual best response, \emph{so the same characterizations support either party}: a user can use our algorithms to find the optimal dataset report that steers the learner toward a less biased hypothesis, or a learner can use them to choose an optimal learning strategy. We adopt the user-side convention throughout for concreteness.

%% file: 3-Possibility-of-Influence.tex
\section{Possibility of Influence}
\label{sec:possibility-of-influence}
\vspace{-3mm}
\subsection{Influential Information Release}
\label{sec:influence}
An equilibrium $(\sigma_r,\sigma_\ell)$ is \emph{influential} if there exist two distinct reports $M,M' \in \mathcal{M}$, both sent with positive probability under $(\sigma_r,\pi)$, such that
the learner responds differently after the two reports: $\sigma_\ell(\cdot\mid M)\neq \sigma_\ell(\cdot\mid M')$. 
Otherwise, the equilibrium is \emph{non-influential}. A non-influential equilibrium always exists: the user sends the same report for every private dataset, so the report conveys no information, and the learner best responds to its prior. The central question is therefore not whether communication can collapse, but whether meaningful communication can survive despite the conflict of interest. In other words, can the user still transmit information that changes the learned hypothesis, even when the user strategically chooses what to reveal?

We say that a learning setting is \emph{influential} if it admits at least one influential equilibrium, and \emph{non-influential} if every equilibrium is non-influential.

\vspace{4mm}
\textit{Problem 1: Given the learning setting $(D, U^r, U^\ell)$, determine whether the setting is influential. }

\paragraph{Conflict over hypotheses.}
For a private dataset $D\in\mathcal D$, let
$\bm{w}_\ell^*(D)\in\arg\max_{\bm w\in\mathcal W}U^\ell(D,h_w)$
and
$\bm w_r^*(D)\in\arg\max_{\bm w\in\mathcal W}U^r(D,h_w)$
denote the learner's and user's ideal hypotheses, respectively. At a high level, the conflict between the learner and the user is a conflict over these ideal hypotheses. To capture this conflict in a tractable and interpretable way, we adopt a standard reduced-form assumption from the game-theoretic literature on strategic communication ~\citep{kamenica2011bayesianpersuasion,crawford1982cheaptalk}, in which the preferences of the two parties differ by a fixed bias vector. 

In particular, we assume that the user's preferred hypothesis is a shifted version of the learner's: $\bm w_r^*(D) = \bm{w}_\ell^*(D) + \bm b$,
where $\bm b \in \mathbb R^d$ encodes the direction and magnitude of the user's strategic bias. This specification captures a systematic misalignment in objectives while preserving the underlying dependence on the private dataset.
This motivates the reduced-form \textbf{ utility functions}
$\widetilde U^\ell(D,\bm w)=-\|\bm w-\bm{w}_\ell^*(D)\|_2^2$
and
$\widetilde U^r(D,\bm w)=-\|\bm w-(\bm w_\ell^*(D)+ \bm b)\|_2^2$,
where $\bm w\in\mathcal W$ is the learner's final hypothesis after observing the user's report.~\footnote{While we place the bias on the user's side, equilibrium depends only on the gap between the two ideal points, so our model applies when the bias instead resides with the learner.}
Each private dataset $D$ induces a learner-optimal hypothesis $\bm w_\ell^*(D)$. Therefore, the prior $\pi$ over datasets induces a corresponding distribution over hypotheses, which we denote by $\pi_{\bm w}$. Concretely, $\pi_{\bm w}$ is the distribution of $\bm w_\ell^*(D)$ when $D \sim \pi$.
Thus, although the primitive uncertainty is over datasets, the private information in our later results is the optimal hypothesis learned from the private dataset $\bm w_\ell^*(D)$.

\paragraph{Reports over hypotheses.}
Since reports are costless and the learner acts only on the induced hypothesis, two datasets inducing the same $\bm w_\ell^*$ are strategically equivalent. The utility-relevant content of a report $M\in\mathcal M$ is therefore the subset of $\mathcal W$ that it identifies. We use report $M$ and region $R\subseteq\mathcal W$ interchangeably: reporting $M$ communicates that $\bm w_\ell^*(D)\in R$, and the learner best-responds given the posterior over $R$. In practice, the user implements report $M$ by submitting any dataset $\tilde D$ such that $\bm w_\ell^*(\tilde D)\in R$; the specific choice is utility-irrelevant since the learner's response depends only on which region $\bm w_\ell^*(\tilde D)$ falls in.

\subsection{Single Dimension}
\label{sec:existence-single-dimention}

Similar to previous works ~\citep{cohen2024strategic,hardt2016strategicclassification}, we adopt the following setup for the user’s
information release problem. We are motivated by screening problems such as school admissions and
hiring, where an individual’s qualification level can be captured via a real-valued number, say, a test score. As we are in a strategic setting, users can “game” the learner by being intentionally vague about their test score, e.g., pass/fail categories.

\paragraph{Setup.} We keep the utility functions exactly as before and consider $d=1$. Therefore, $U^\ell(D,w)=-|w-w_\ell^*(D)|^2$ and  $U^r(D,w)=-|w-(w_\ell^*(D)+b)|^2$ with $b>0$. Here $w_\ell^*(D)$ is the optimal hypothesis learned from the private dataset, and $w$ is the final or deployed hypothesis. 

\begin{theorem}
\label{thm:existence-onde-dimension}
Suppose $w_\ell^*(D)\in[0,1]$ has a continuous prior $\pi_w$ with full support on $[0,1]$. For any $t\in(0,1)$, define $\Phi_L(t)=\mathbb E\!\left[w_\ell^*(D)\mid w_\ell^*(D)\in[0,t]\right]$ and $\Phi_H(t)=\mathbb E\!\left[w_\ell^*(D)\mid w_\ell^*(D)\in(t,1]\right]$.
Then learning is influential if and only if there exists $t \in (0,1)$ such that $t + b = \frac{\Phi_L(t) + \Phi_H(t)}{2}$.

\end{theorem}

\paragraph{Statistics oracle.}
Since we work with an arbitrary prior $\pi_w$, we assume access to an oracle that,
given a measurable region in the parameter space $R\subseteq\mathcal W$, returns its prior mass, posterior mean, and posterior second moment (Algorithm~\ref{algo:posterior-oracle}).
The region $R$ need not be continuous. For certain priors considered later, these statistics
are available in closed form, so the oracle abstraction is not needed. When $d>1$,
we use the same oracle with the bold notation $\pi_{\bm w}$ and $\bm w_\ell^*(D)$.

\begin{algorithm}[t]
\footnotesize
\caption{$\mathrm{Oracle}(\pi_w,R)$}
\label{algo:posterior-oracle}
\begin{algorithmic}[1]
\State \textbf{Input:} prior $\pi_w$, measurable region $R\subseteq\mathcal W$ with positive prior mass
\State \textbf{Output:} prior mass $p(R)$, posterior mean $\Phi(R)$, posterior second moment $\Psi(R)$
\State $p(R) \leftarrow \Pr_{\pi_w}\!\bigl(w_\ell^*(D)\in R\bigr)$
\State $\Phi(R) \leftarrow \mathbb E_{\pi_w}\!\left[w_\ell^*(D)\mid w_\ell^*(D)\in R\right]$
\State $\Psi(R) \leftarrow \mathbb E_{\pi_w}\!\left[w_\ell^*(D)w_\ell^*(D)^\top\mid w_\ell^*(D)\in R\right]$
\State \Return $(p(R),\Phi(R),\Psi(R))$
\end{algorithmic}
\end{algorithm}

\paragraph{Algorithm.}
Theorem~\ref{thm:existence-onde-dimension} reduces influence detection to finding a root of $g(b,t)
=
\frac{\Phi(R(t))+\Phi(R'(t))}{2}
-t-b,$
where \(R(t)=[0,t]\) and \(R'(t)=(t,1]\). Each evaluation of \(g\) requires two oracle calls, one for each region. Therefore, if a solver evaluates \(g\) at \(K\) candidate points, the oracle complexity is \(2K\).

\begin{remark}
When the cutoff is computed numerically, the correctness of the algorithm depends on the guarantees of the solver and the error in the
posterior means returned by the oracle. For specific priors, the cutoff condition can be
solved analytically, as shown next.
\end{remark}

\begin{corollary}
\label{corollary:uniform-one-dimentional}
If the prior $\pi_w$ is uniform on $[0,1]$, learning is influential if and only if $b<\frac{1}{4}.$
\end{corollary}

 We can check this condition in constant time. Informally, when the bias is sufficiently small, the user's preferences are sufficiently aligned with the learner's so that partial information transmission can be sustained: low parameters prefer to induce lower hypotheses, and high parameters prefer to induce higher ones, yielding an influential equilibrium. In contrast, when $b$ gets large, all parameters prefer to induce a higher hypothesis regardless of the dataset. All missing proofs are in Appendix~\ref{sec:appendix-possibility}.

\subsection{Multiple Dimensions}
We now extend the earlier setup to higher dimensions $d > 1$. Each private dataset $D \sim \pi$ results in a learner-optimal parameter \(\bm w_\ell^*(D)\in\mathbb R^d\), and hence the dataset prior induces a prior \(\pi_{\bm w}\) over optimal parameters. We keep the same utility functions as before and give a sufficient condition.

\begin{theorem}
\label{thm:spherical-md-existence}
Suppose the prior \(\pi_{\bm w}\) is non-degenerate with finite first moment, and is spherically symmetric around its mean \(\bar{\bm w}:=\mathbb E_{\pi_{\bm w}}[\bm w_\ell^*(D)]\), learning is influential for every \(\bm b\in\mathbb R^d\).
\end{theorem}

This applies, for example, to \(\mathcal N(\bar{\bm w},\sigma^2 I_d)\) and to uniform distributions on balls or spheres centered at \(\bar{\bm w}\). Spherical symmetry lets us split the prior through its mean along any direction orthogonal to \(\bm b\), yielding two regions with distinct posterior means that are unaffected by the bias. The finite first-moment assumption ensures that \(\bar{\bm w}\) exists, and non-degeneracy rules out point masses, so the equilibrium is influential. For general multidimensional priors, however, there is no scalar cutoff equation as in one dimension, making the problem hard. We extend this discussion in Section~\ref{sec:user-optimal-multiple}. 

%% file: 4-Optimal-Influence.tex
\vspace{-2mm}
\section{The User's Optimal Information Release Problem}
\label{sec:optimal-information-release}

In this section, we study the user’s optimization problem. 
Given the learner’s strategy, the user seeks to choose a dataset report \(M \in \mathcal M\) that maximizes their utility \(U^r\). A naive algorithm would evaluate the user’s utility under every feasible dataset report and output a maximizer. Such an approach is computationally inefficient in general, because the report space \(\mathcal M\) can be exponentially large or even infinite. Our goal is therefore to identify structures in the problem that yield efficient algorithms for computing an optimal dataset report. We provide both positive and negative results.

\vspace{4mm}
\textit{Problem 2: Given the learning setting $(D, U^r, U^\ell)$, find the report that maximizes the user's utility. }


\subsection{Single Dimension}
\label{sec:single-dimension-optimal}
We consider the quadratic utility functions as in Theorem~\ref{thm:existence-onde-dimension} for \(d=1\). Let
$w_\ell^*(D)\in \mathcal W=[\underline w,\bar w]\subseteq \mathbb R$
and let \(\pi_w\) denote the prior over \(w_\ell^*(D)\).

\begin{lemma}
\label{lem:ordered-regions-1d}
When \(d=1\), after identifying dataset reports that result in the same learner hypothesis, the
equilibrium dataset reports are ordered by \(w_\ell^*(D)\).
\end{lemma}

Lemma~\ref{lem:ordered-regions-1d} says that any one-dimensional equilibrium
can be represented by ordered report regions. Thus, if an equilibrium has \(n\)
distinct reports, we can write its regions as a partition $\underline w=t_0<t_1<\cdots<t_n=\bar w .$
The next lemma shows why our bounded-support assumption on the prior is useful: it gives a
finite upper bound on the number of partitions, and hence on the number of reports
that we should consider to find the user-optimal dataset report.

\begin{lemma}
\label{lem:finite-bound-1d}
Under a continuous full-support prior $\pi_w$ on \([\underline w,\bar w]\) and $b>0$, every equilibrium has at most
$N_{\max}
=
\left\lfloor 1+\frac{\bar w-\underline w}{2b}\right\rfloor
$
distinct dataset reports.
\end{lemma}

The two lemmas reduce the problem of finding the user's optimal dataset report to
a finite search over ordered partitions. Lemma~\ref{lem:ordered-regions-1d}
shows that equilibrium regions can be ordered, and
Lemma~\ref{lem:finite-bound-1d} bounds the number of partitions that need to be considered.

\begin{theorem}
\label{theorem:1d-feasibility}
Assume \(\pi_w\) is continuous with full support on
\([\underline w,\bar w]\), and let \(b>0\). A feasible one-dimensional equilibrium
is any partition \(\underline w=t_0<\cdots<t_n=\bar w\), with \(n\le N_{\max}\),
satisfying
\begin{align}
t_i
&=\frac{\Phi(t_{i-1},t_i)+\Phi(t_i,t_{i+1})}{2}-b,
\qquad i=1,\ldots,n-1 ,
\label{eq:cutpoint-system}\\
\min_{\substack{\underline w=t_0<\cdots<t_n=\bar w\\ n\le N_{\max}}}
&\sum_{i=1}^n
p(t_{i-1},t_i)
\left(\Psi(t_{i-1},t_i)-\Phi(t_{i-1},t_i)^2\right)
\quad
\mathrm{s.t.}\ \eqref{eq:cutpoint-system}.
\label{eq:user-optimal-objective}
\end{align}
The user-optimal equilibrium is any feasible partition solving
\eqref{eq:user-optimal-objective}. Moreover, if
\(w_\ell^*(D)\in[t_{i-1},t_i]\), then the user-optimal dataset report is \(M^*(D)=[t_{i-1},t_i]\).
\end{theorem}

Informally, every equilibrium is a monotone partition of
\([\underline w,\bar w]\). For any such partition, the learner responds on each
partition with the posterior mean of \(w_\ell^*(D)\). Under the quadratic user
utility, the user's expected loss from a report region is the posterior
variance within that region plus the fixed bias cost \(b^2\). Therefore, a
finer partition weakly provides higher utility to the user because it weakly reduces
within-region uncertainty. Hence, the user's optimization problem reduces to
finding the finest feasible partition and then releasing information about the unique region
that contains \(w_\ell^*(D)\).

\begin{corollary}
\label{cor:finite-prior-dp}
Suppose \(\pi_w\) has finite support
\(w^{(1)}<\cdots<w^{(m)}\). A feasible finite-support equilibrium is a partition
of the ordered support into contiguous blocks
$
R_i=\{w^{(t_{i-1}+1)},\ldots,w^{(t_i)}\},
0=t_0<t_1<\cdots<t_k=m .
$ For adjacent blocks \(R_i\) and \(R_{i+1}\), the cutpoint equation
\eqref{eq:cutpoint-system} is replaced by
\begin{equation}
w^{(t_i)}
\le
\frac{\Phi(R_i)+\Phi(R_{i+1})}{2}-b
\le
w^{(t_i+1)} .
\label{eq:block-ic}
\end{equation}
\end{corollary}

\paragraph{Algorithm.}
We provide the pseudocode of Algorithm~\ref{algorithm:one-d-optimal} in Appendix~\ref{section:appendix-optimal-release}. For finite priors, the algorithm uses a dynamic program based on Corollary~\ref{cor:finite-prior-dp}. The support is first sorted as
\(w^{(1)}<\cdots<w^{(K)}\). The algorithm uses Algorithm~\ref{algo:posterior-oracle} to get the statistics for
every contiguous block, checks the adjacent-block constraints
\eqref{eq:block-ic}, and then solves the shortest-path problem with block costs given $\sum_i p(R_i)\left(\Psi(R_i)-\Phi(R_i)^2\right)$. There are \(O(m^2)\) contiguous blocks
and \(O(m^3)\) adjacent block pairs, so the exact finite-prior implementation runs
in \(O(m^3)\) time. 

For continuous priors, Lemma~\ref{lem:finite-bound-1d} bounds the search by
\(N_{\max}\). For each \(n\le N_{\max}\), the algorithm searches for feasible
solutions of \eqref{eq:cutpoint-system} with \(t_0=\underline w\) and
\(t_n=\bar w\). Each residual evaluation uses posterior statistics for \(n\)
partitions, hence \(O(n)\) oracle calls. If the solver uses \(I_n\) residual
evaluations and returns \(S_n\) feasible partitions, the oracle complexity for
that \(n\) is \(O(n(I_n+S_n))\). Thus the total oracle complexity is
$O\!\left(\sum_{n=1}^{N_{\max}} n(I_n+S_n)\right),$
plus the solver's internal arithmetic cost. Once an optimal partition is stored,
report lookup takes \(O(\log n)\) time by binary search.

\begin{remark}
For finite priors, the dynamic program gives an exact certificate of optimality.
For general continuous priors, there is no distribution-free guarantee without
specifying a global solver for Eq. \eqref{eq:cutpoint-system}. If the solver only finds
local or approximate solutions, then the returned report is optimal only over the
feasible solutions found, up to the solver tolerance.
\end{remark}

\begin{proposition}
\label{prop:efficient-one-dimentional}
Suppose $\pi_w$ is continuous and uniform on $[0,1]$. Then the cutpoint program
\eqref{eq:cutpoint-system} admits the closed-form solution $t_i=\frac{i}{n^*}-2b\,i(n^*-i),
\quad \text{ for }i=0,\dots,n^*,$
where
$
n^*=\max\{n\in\mathbb N:\;2b\,n(n-1)<1\}.$ Hence, the user-optimal equilibrium partition is $[0,1]=\bigcup_{i=1}^{n^*}[t_{i-1},t_i].$ Moreover if $w_\ell^*(D)\in[t_{i-1},t_i]$, then the user-optimal report is $[t_{i-1},t_i]$.
\end{proposition}

\paragraph{Efficient Algorithm.} For uniform priors, posterior means are linear in the cutpoints, so the equilibrium conditions reduce to a linear equation with a closed-form solution. By Proposition~\ref{prop:efficient-one-dimentional}, after computing \(w_\ell^*(D)\), the user-optimal report can be found in \(O(\log n^*)\) time by binary search.

\subsection{Multiple Dimensions}
\label{sec:user-optimal-multiple}

We now consider the same setting as Section~\ref{sec:single-dimension-optimal} for \(d>1\). Each private dataset \(D\) results in an optimal hypothesis $\bm w_\ell^*(D)\in\mathcal W\subseteq\mathbb R^d$,
and the dataset prior induces a prior \(\pi_{\bm w}\) over \(\bm w_\ell^*(D)\). In multiple dimensions, equilibria are again represented by partitions of \(\mathcal W\), but unlike the one-dimensional case, there is no canonical ordering of partitions.

\paragraph{Problem:} Let \(\mathcal R=\{R_1,\dots,R_k\}\) be a measurable partition of \(\mathcal W\), and
let \(\Phi(R_i)=\mathbb E_{\pi_{\bm w}}[\bm w_\ell^*(D)\mid
\bm w_\ell^*(D)\in R_i]\). Then \(\mathcal R\) is an equilibrium partition if and
only if, up to tie boundaries,
\begin{equation}
R_i
=
\left\{
\bm w\in\mathcal W:
\|\Phi(R_i)-(\bm w+\bm b)\|_2^2
\le
\|\Phi(R_j)-(\bm w+\bm b)\|_2^2
\text{ for all }j
\right\}.
\label{eq:multidim-fixed-point}
\end{equation}
Among all equilibrium partitions, the user-optimal partition solves
\begin{equation}
\max_{\mathcal R\ \text{satisfying Eq. }\eqref{eq:multidim-fixed-point}}
\sum_{i=1}^k
\Pr_{\pi_{\bm w}}(\bm w_\ell^*(D)\in R_i)
\mathbb E_{\pi_{\bm w}}\!\left[
-\|\Phi(R_i)-(\bm w_\ell^*(D)+\bm b)\|_2^2
\mid \bm w_\ell^*(D)\in R_i
\right].
\label{eq:multidim-user-optimal}
\end{equation}
If \(\mathcal R^*=\{R_1^*,\dots,R_{k^*}^*\}\) is user-optimal, then
the user-optimal multi-dimensional dataset report is \(M^*(D)=R_i^*\) whenever \(\bm w_\ell^*(D)\in R_i^*\). 

We show that the user's optimal information-release problem is NP-hard for arbitrary priors, via a reduction from the NP-hard \textit{planar $k$-means} problem. We also show that exact evaluation of the statistics oracle under general priors is \#P-hard by considering a uniform $[0,1]^d$ prior. The formal hardness statements and their proofs are provided in Appendix~\ref{sec:appendix-np-hard}.

\subsubsection{Efficient Algorithm for Multiple Dimensions}
\label{sec:efficient-multidimentional-optimal}
Given the hardness of the problem for arbitrary prior distributions, we identify a
structured class of priors for which the strategic part of the problem becomes
tractable. The key observation is that the bias vector \(\bm b\) determines the direction of disagreement between the user and the learner. Orthogonal to this
direction, their objectives are aligned. All proofs of this section are provided in Appendix~\ref{sec:appendix-efficient-optimal-multi-dimension}.



\paragraph{Factorized priors.}
Assume the bias in the user's utility function is \(\bm b\neq \mathbf 0\), and let
\(\hat{\bm b}=\bm b/\|\bm b\|_2\) denote the unit bias direction. Since the user prefers the shifted hypothesis
\(\bm w_\ell^*(D)+\bm b\), the direction \(\hat{\bm b}\) captures the strategic conflict between the user and the learner. We can rewrite
\(\bm w_\ell^*(D)=s\hat{\bm b}+\mathbf z\), where
\(s=\hat{\bm b}^{\top}\bm w_\ell^*(D)\) and
\(\mathbf z=\bm w_\ell^*(D)-s\hat{\bm b}\). We assume \(s\in[\underline s,\bar s]\subseteq\mathbb R\), and call
\(s\) the conflict projection and \(\mathbf z\) the agreement component.
Here, \(s\) is the scalar projection of \(\bm w_\ell^*(D)\) along the bias
direction, while \(\mathbf z\) is the component orthogonal to that direction. Let \(\pi_{\bm w}\) denote the prior over \(\bm w_\ell^*(D)\), and let
\(\pi_s\) and \(\pi_{\mathbf z}\) be the marginal priors over \(s\) and
\(\mathbf z\). The factorized-prior assumption is
\(\pi_{\bm w}=\pi_s\otimes\pi_{\bm z}\), equivalently, \(s\) and \(\mathbf z\) are independent under the prior. 

\begin{remark}
The factorized prior assumption is a rotated product-prior assumption and therefore it does not restrict the original coordinates of \(\bm w_\ell^*(D)\) to be independent.
\end{remark}

\begin{theorem}
\label{thm:factorized-lift}
Under the setting above, any one-dimensional equilibrium partition of \(s\) with
bias \(\beta=\|\bm b\|_2\) lifts to a multidimensional equilibrium that reveals
\(\mathbf z\) exactly and manipulates only along \(\hat{\bm b}\). If the scalar partition is user-optimal,
then its lift is user-optimal among lifted equilibria.
\end{theorem}

\paragraph{Algorithm.}
Given a private dataset \(D\), we first compute the learner-optimal parameter \(\bm w_\ell^*(D)\).
Then decompose it into the conflict projection \(s=\hat{\bm b}^{\top}\bm w_\ell^*(D)\)
and the agreement component
\(\mathbf z=\bm w_\ell^*(D)-s\hat{\bm b}\). Then we can use Algorithm~\ref{algorithm:one-d-optimal} on the scalar coordinate \(s\) with bias
\(\beta=\|\bm b\|_2\). If the resulting one-dimensional equilibrium partition is
\(t_0<t_1<\cdots<t_{n^*}\), find the unique partition \([t_{i-1},t_i]\) containing \(s\).
Return the report $M^*(D)=
\left\{
\bm w\in\mathcal W:
\hat{\bm b}^{\top}\bm w\in[t_{i-1},t_i]
\right\}.$
Computing \(s\) and \(\mathbf z\) takes \(O(d)\) time. After the one-dimensional cutpoints
are computed, locating the partition containing \(s\) takes \(O(\log n)\) time by binary
search, as in Algorithm~\ref{algorithm:one-d-optimal}.  Hence,
after computing \(\bm w_\ell^*(D)\), the report can be computed in
\(O(d+\log n)\) time. In the uniform continuous prior case,
Corollary~\ref{corollary:factorized-uniform} in Appendix~\ref{sec:appendix-efficient-optimal-multi-dimension} gives the cutpoints in closed form.

\paragraph{Equilibrium certificate.}
When the factorization assumption is relaxed, Proposition~\ref{prop:eps-equilibrium}
provides a finite-prior certificate for the returned partition. If
$\varepsilon_{\mathrm{IC}}=0$, the partition is an exact equilibrium on the support
of $\pi_{\bm w}$; if $\varepsilon_{\mathrm{IC}}\le \varepsilon$, it is an
$\varepsilon$-equilibrium on that support.

\begin{proposition}
\label{prop:eps-equilibrium}
Suppose $\pi_{\bm w}$ has finite support $\{\bm w_1,\ldots,\bm w_n\}$. Let
$R_1,\ldots,R_k$ be the regions returned by the algorithm, and define
\[
\varepsilon_{\mathrm{IC}}
=
\max_{i\in[k]}\max_{\ell\neq i}\max_{\bm w_j\in R_i}
\left[
\|\Phi(R_i)-(\bm w_j+\bm b)\|_2^2
-
\|\Phi(R_\ell)-(\bm w_j+\bm b)\|_2^2
\right]_+ .
\]
Then the returned partition is an $\varepsilon_{\mathrm{IC}}$-equilibrium on the
support of $\pi_{\bm w}$. The value $\varepsilon_{\mathrm{IC}}$ is computable in
$O(nkd)$ time using $\mathrm{Oracle}(\pi_{\bm w},R_i)$.
\end{proposition}

%% file: 5-Noise-In-Learning.tex
\section{Noise in Learning}
\label{sec:noise-in-learning}

So far, we have investigated imperfect information release as a result of conflicting incentives of the user and the learner. We now add an additional source of imperfect information: noise in the learning process. In a loan-scoring application, the user may choose a dataset report in order to convince the learner to learn a more favorable scoring rule. However, the learner may not train on exactly the dataset report that the user intended. Stochastic training, privacy noise, or noisy evaluation can make the learner behave as if it had trained on a perturbed dataset. Thus, noise in learning can be viewed equivalently as perfect learning but from a noisy dataset. In this section, we investigate these settings and propose algorithms for computing the user's optimal dataset report.

\paragraph{Model under Noise.}
To investigate strategic information release under noise, we extend the framework of Section~\ref{sec:framework}. The user sends a dataset report \(M \in \mathcal M\), \textit{but the learner may use a noisy interpretation} \(\tilde M \in \widetilde{\mathcal M}\). When the report space is Euclidean, say \(\mathcal M \subseteq \mathbb R^d\), one natural example is
\(\tilde M=M+\boldsymbol{\varepsilon}\), where \(\boldsymbol\varepsilon\) is a zero mean Gaussian perturbation. More generally, conditional on \(M\), the learner receives \(\tilde M\) according to an interpretation rule \(G(\tilde M\mid M)\).

The key difference from the noiseless setup is that the learner no longer conditions on the exact report \(M\). Instead, it conditions only on the noisy report \(\tilde M\), and therefore cannot, in general, distinguish \(M\) from other reports \(M'\) for which \(G(\tilde M\mid M')>0\). Thus, the learner's posterior belief is \(\Phi(D \mid \tilde M)\), the learner's strategy is \(\sigma_\ell(h \mid \tilde M)\), and the user's strategy remains \(\sigma_r(M \mid D)\). For a fixed private dataset \(D\), the user's expected utility from sending report \(M\) is
\(\sum_{\tilde M \in \widetilde{\mathcal M}} G(\tilde M\mid M)
\sum_{h\in\mathcal H} U^r(D,h)\sigma_\ell(h\mid \tilde M)\).
For a fixed noisy report \(\tilde M\), the learner's expected utility from choosing hypothesis \(h\) is
\(\mathbb E_{D'\sim \Phi(\cdot\mid \tilde M)}[U^\ell(D',h)]
=
\sum_{D' \in \mathcal D} U^\ell(D',h)\Phi(D' \mid \tilde M)\).
The definition of equilibrium in the noisy setting follows after updating the expected utilities in Definition~\ref{definition:equilibrium}.

\paragraph{Noise changes equilibrium structure.} For algorithmic results, we specialize to additive Gaussian noise, $\tilde M = M+\varepsilon,
\varepsilon\sim\mathcal N(0,\sigma^2).$
In the noiseless setting, each report $M$ induces a single posterior, so equilibrium regions are ordered intervals, as in Section~\ref{sec:optimal-information-release}. Under noise, the learner cannot distinguish $M$ from reports $M'$ with $G(\tilde M\mid M')>0$, so the interval-partition structure no longer applies. The equilibrium object is instead the full report \emph{tuple}
$\mathbf M=(M_1,\ldots,M_n)$
over the finite support of $\pi_w$. We show that any monotone influential noisy equilibrium has an \emph{endpoint-pooling} structure: there exist $t,t'\ge 0$ such that
\[
M_1=\cdots=M_t=0
<
M_{t+1}<\cdots<M_{n-t'}
<
M_{n-t'+1}=\cdots=M_n=1.
\]
Thus, distinct true parameters send the same report (pool) only at the boundaries $\{0,1\}$, as shown in Proposition~\ref{prop:noisy-equilibrium-structure} and Appendix~\ref{sec:appendix-noise}. This gives the feasible set $\mathcal E_{\mathrm{noisy}}$ of monotone influential tuples.

\paragraph{Noise increases algorithmic complexity.}
In the noiseless case, a report near a partition boundary produces a sharp learner response. Under noise, the learner cannot sharply distinguish nearby reports. This means the user-optimal tuple cannot be found by a finest-partition argument as in Section~\ref{sec:optimal-information-release}. Instead, it requires maximizing over verified equilibrium tuples.

\begin{theorem}
\label{thm:user-optimal-noisy}
If $\mathcal E_{\mathrm{noisy}}\neq\emptyset$, then the user-optimal monotone noisy equilibrium is
\vspace{-2mm}
\[
\begin{aligned}
\mathbf M^*
&\in
\arg\max_{\mathbf M\in\mathcal E_{\mathrm{noisy}}}
\sum_{i=1}^n \pi_i V_i(M_i;\mathbf M),
&
V_i(r;\mathbf M)
&=
\int_{\mathbb R}
-\Bigl(w^{(i)}+b-\Phi(\tilde M;\mathbf M)\Bigr)^2
G(\tilde M\mid r)\,d\tilde M .
\end{aligned}
\]
If $w_\ell^*(D)=w^{(i)}$, the user-optimal dataset report is $M^*(D)=M_i^*$.
\end{theorem}
The complete exposition of this section, including the full algorithm, complexity analysis, binary closed-form case, and multidimensional extension, is provided in Appendix~\ref{sec:appendix-noise}.

%% file: 5-Experiments.tex
\section{Experimental Evaluation}
\label{sec:experiment}
We evaluate our methods on four real-world datasets using two off-the-shelf learner classes with default settings: linear SVM and a two-layer neural network (MLP). Specifically, we evaluate the efficient multidimensional algorithms for computing the user-optimal dataset report in both noiseless and noisy settings (Section~\ref{sec:efficient-multidimentional-optimal} and Appendix~\ref{sec:noisy-multiple-dimensions}), varying bias and noise across four levels with three independent repetitions per configuration. Throughout, we use \textit{empirical priors} estimated from $50$ bootstrap resamples of the data. Details on the experimental setting are in Appendix~\ref{sec:appendix-eperiment-setup}.

\begin{table}[ht]
\small
\centering
\caption{Runtime for noisy and noiseless multi-dimensional algorithms}
\label{tab:dataset-runtime}
\begin{tabular}{llrrr}
\hline
\textbf{Model} & \textbf{Dataset} & \(\dim(\bm w_\ell^*)\) & \textbf{Noisy Time (s)} & \textbf{Noiseless Time (s)} \\
\hline
\multirow{4}{*}{Linear SVM}
& Credit Approval  & 47       & $506.4 \pm 187.9$ & $0.012 \pm 0.007$ \\
& Census           & 104--106 & $550.8 \pm 160.5$ & $0.107 \pm 0.006$ \\
& School Admission & 52--55   & $455.0 \pm 230.7$ & $0.064 \pm 0.005$ \\
& Prosper Loan     & 2090     & $786.5 \pm 182.3$ & $110.415 \pm 11.536$ \\
\hline
\multirow{4}{*}{MLP}
& Credit Approval  & 1570      & $572.6 \pm 165.8$ & $0.430 \pm 0.386$ \\
& Census           & 3458      & $527.2 \pm 171.3$ & $5.623 \pm 0.833$ \\
& School Admission & 1762--1826 & $511.1 \pm 130.6$ & $3.275 \pm 1.549$ \\
& Prosper Loan     & 6443      & $543.5 \pm 198.6$ & $52.266 \pm 6.117$ \\
\hline
\end{tabular}
\end{table}

\paragraph{Scalability of Algorithms.}
Table~\ref{tab:dataset-runtime} reports end-to-end runtime for both algorithms. Here \(\dim(\bm w_\ell^*)\) denotes the number of learned model parameters.
The noiseless algorithm dynamic program subroutine adds negligible overhead beyond model training across all settings, and this remains consistent on scaling the size of the prior (Appendix~\ref{sec:additional-experiment-results}). In comparison, the noisy algorithm has a larger complexity because it enumerates \(O(n^2)\) patterns
(Proposition~\ref{prop:noisy-equilibrium-structure}) and solves a nonlinear
non-convex system per pattern. Due to the numerical nature of the objective, we verify approximation errors using equilibrium residuals and best-response violations. These violations are at most
\(5\times 10^{-6}\). We provide the diagnostic details in
Appendix~\ref{sec:additional-experiment-results}.


\begin{table}[ht]
\small
\centering
\caption{Utility gain of strategic user strategy over non-strategic under varying bias and noise.}
\label{tab:utility-gain-noiseless-noisy}
\begin{tabular}{llrrrr}
\hline
\textbf{Model} & \textbf{Dataset}
& \multicolumn{2}{c}{\textbf{Bias}}
& \multicolumn{2}{c}{\textbf{Noise}} \\
\cline{3-6}
& & \textbf{Utility Gain} & \textbf{Win Rate}
  & \textbf{Utility Gain} & \textbf{Win Rate} \\
\hline
\multirow{4}{*}{Linear SVM}
& Credit Approval   & $21.1 \pm 19.2$  & 1.00 & $5.9 \pm 6.5$    & 0.88 \\
& Census            & $0.6 \pm 4.0$    & 0.67 & $-3.3 \pm 3.8$   & 0.00 \\
& School Admission  & $0.2 \pm 0.3$    & 0.50 & $0.0 \pm 0.0$    & 0.50 \\
& Prosper Loan      & $0.3 \pm 1.3$    & 0.67 & $0.8 \pm 0.6$    & 0.88 \\
\hline
\multirow{4}{*}{MLP}
& Credit Approval   & $3.0 \pm 7.8$    & 0.75 & $4.4 \pm 6.2$    & 0.88 \\
& Census            & $4.1 \pm 9.3$    & 0.67 & $-11.2 \pm 21.3$ & 0.12 \\
& School Admission  & $6.0 \pm 10.1$   & 0.75 & $4.7 \pm 11.6$   & 0.50 \\
& Prosper Loan      & $23.3 \pm 30.1$  & 0.83 & $59.9 \pm 273.4$ & 0.50 \\
\hline
\end{tabular}
\end{table}

\paragraph{Effect of Bias and Noise on User Utility.}
Table~\ref{tab:utility-gain-noiseless-noisy} reports the mean utility gain,
defined as strategic minus non-strategic utility and scaled by \(10^3\), together with the win rate, i.e., the fraction of $12$ trials in which the strategic user strategy strictly outperforms the non-strategic strategy. Here, \emph{non-strategic} refers to a strategy that leads to an uninfluential equilibrium, where the learner ignores the report and only uses its prior (Section~\ref{sec:possibility-of-influence}). Across different bias levels, the strategic strategy yields positive mean gains across all dataset-model combinations, with all win rates at least \(0.50\). For noise, the results are more heterogeneous. Credit Approval benefits under noise for both model classes, with win rates \(0.88\). Prosper Loan also benefits for Linear SVM, with gain \(0.8\) and win rate \(0.88\). Census is the persistent exception: gains are negative, and win rates fall to \(0.00\) for Linear SVM and \(0.12\) for MLP. This is consistent with large violations of the equilibrium certificate ($\varepsilon_{\mathrm{IC}}$ up to 28.7, Appendix~\ref{sec:additional-experiment-results}) based on Proposition~\ref{prop:eps-equilibrium}. Overall, bias is a more reliable driver of strategic benefit than noise. We provide additional experiment results on factorized prior assumption and approximation errors in Appendix~\ref{sec:additional-experiment-results}.

%% file: Appendix.tex
\section{Limitations, Future Work and Broader Impact}
\label{sec:appendix-limitation-future-impact}

\textbf{Limitations and Future Work.}
We assume that both the user and the learner know the prior over learner-optimal hypotheses. This is standard in strategic communication in game theory, and strategic classification~\cite{kamenica2011bayesianpersuasion,crawford1982cheaptalk,cohen2024strategic}. In real deployments, the learner may only partially estimate prior from past interactions, and users may hold heterogeneous beliefs about the learner's model. Both issues could affect the optimal information-release strategy.
Given the hardness of the general multidimensional problem, our algorithms also rely on structural assumptions on the prior. Identifying other tractable cases is a natural next step.

\paragraph{Broader Impact:} This work studies a fundamental tension in training ML systems: learners and users may have misaligned objectives. On the plus side, by formalizing this conflict and developing algorithms for strategic user communication, we provide tools that help users protect their interests without relying on the learner to voluntarily constrain itself — a requirement that current regulatory approaches struggle to enforce~\cite{yeh2025position,10.1145/3533379}. The main societal benefit is that protection shifts from the provider side, where fairness or debiasing constraints may be difficult to enforce, to the user side, where users are equipped with strategic communication tools. This is especially relevant in high-stakes domains such as lending, hiring, and content recommendation.
A potential negative use is dual application: our model could be used by learners to choose their optimal strategies that offset the strategic user.

\section{Missing Proofs of Section~\ref{sec:possibility-of-influence}}
\label{sec:appendix-possibility}

\begin{proof}[\textbf{Proof of \autoref{thm:existence-onde-dimension}}]
For any measurable set $A\subseteq[0,1]$ with positive $\pi_w$-probability, write
\[
\Phi(A):=\mathbb E_{\pi_w}\!\left[w_\ell^*(D)\mid w_\ell^*(D)\in A\right].
\]
Thus $\Phi_L(t)=\Phi([0,t])$ and $\Phi_H(t)=\Phi((t,1])$.

\textbf{($\Leftarrow$)} Fix $t\in(0,1)$ with
$t+b=(\Phi([0,t])+\Phi((t,1]))/2$. Consider the user strategy with two report regions
$R_1=[0,t]$ and $R_2=(t,1]$. Since $\pi_w$ has full support and is atomless,
$\pi_w(R_1),\pi_w(R_2)>0$, so the learner's best responses are $\Phi(R_1)$ and
$\Phi(R_2)$, with $\Phi(R_1)<\Phi(R_2)$.

For a dataset $D$, the utility gain from reporting $R_2$ rather than $R_1$ is
\[
\bigl(\Phi(R_2)-\Phi(R_1)\bigr)
\bigl(2(w_\ell^*(D)+b)-\Phi(R_1)-\Phi(R_2)\bigr).
\]
Because $\Phi(R_2)>\Phi(R_1)$, the sign of this expression is the sign of
$w_\ell^*(D)+b-(\Phi(R_1)+\Phi(R_2))/2$. By the cutoff condition, this is the sign of
$w_\ell^*(D)-t$. Hence datasets with $w_\ell^*(D)<t$ strictly prefer report region
$R_1$, datasets with $w_\ell^*(D)>t$ strictly prefer report region $R_2$, and the
boundary value $w_\ell^*(D)=t$ is indifferent. Therefore the reporting strategy is
satisfies the user's no-deviation condition. Since the two reports induce distinct learner responses, learning is
influential.

\textbf{($\Rightarrow$)} Suppose learning is influential in some equilibrium. Since the
learner's utility is quadratic, the learner's best response to any on-path report region
$R$ is the posterior mean $\Phi(R)$.

 For any two report regions $R,R'$ with
$\Phi(R)<\Phi(R')$, the utility gain from reporting $R'$ rather than $R$ is
\[
\bigl(\Phi(R')-\Phi(R)\bigr)
\bigl(2(w_\ell^*(D)+b)-\Phi(R)-\Phi(R')\bigr),
\]
which is strictly increasing in $w_\ell^*(D)$. Hence there is a unique crossing threshold
$\tau=(\Phi(R)+\Phi(R'))/2-b$. Therefore, after merging reports that induce the same
learner response and ignoring null sets, equilibrium report regions form an ordered partition $0=t_0<t_1<\cdots<t_n=1$ with $n\ge2$, where
$R_i=(t_{i-1},t_i]$ and the learner responds with $\Phi(R_i)$ to each region.
 At each interior cutpoint $t_i$, the boundary parameter
$w_\ell^*(D)=t_i$ must be indifferent between the adjacent learner responses
$\Phi(R_i)$ and $\Phi(R_{i+1})$. Hence
\[
t_i+b=\frac{\Phi(R_i)+\Phi(R_{i+1})}{2},
\qquad i=1,\ldots,n-1. \tag{$*$}
\]

Define
$g(t):=(\Phi([0,t])+\Phi((t,1]))/2-t-b$. The function $g$ is continuous on $(0,1)$.

At $t_1$, $\Phi([0,t_1])=\Phi(R_1)$, while $\Phi((t_1,1])$ is a weighted average of
$\Phi(R_2),\ldots,\Phi(R_n)$, so $\Phi((t_1,1])\ge \Phi(R_2)$. Therefore, by $(*)$,
\[
g(t_1)\ge \frac{\Phi(R_1)+\Phi(R_2)}{2}-t_1-b=0.
\]

At $t_{n-1}$, $\Phi((t_{n-1},1])=\Phi(R_n)$, while $\Phi([0,t_{n-1}])$ is a weighted
average of $\Phi(R_1),\ldots,\Phi(R_{n-1})$, so
$\Phi([0,t_{n-1}])\le \Phi(R_{n-1})$. Therefore, by $(*)$,
\[
g(t_{n-1})\le \frac{\Phi(R_{n-1})+\Phi(R_n)}{2}-t_{n-1}-b=0.
\]

Since $g$ is continuous and $g(t_1)\ge0\ge g(t_{n-1})$, the intermediate value theorem
gives some $t^*\in[t_1,t_{n-1}]\subset(0,1)$ such that $g(t^*)=0$. Equivalently,
\[
t^*+b=\frac{\Phi([0,t^*])+\Phi((t^*,1])}{2}.
\]
Thus there exists $t^*\in(0,1)$ satisfying the stated condition.
\end{proof}

\begin{proof} [\textbf{Proof of Corollary ~\autoref{corollary:uniform-one-dimentional}}]
Under the uniform prior, $\Phi_L(t)=\frac{t}{2}$ and $\Phi_H(t)=\frac{1+t}{2}.$
Substituting into the cutoff condition from Theorem~\ref{thm:existence-onde-dimension} gives $t+b=\frac{1}{2}\left(\frac{t}{2}+\frac{1+t}{2}\right)
=\frac{1+2t}{4}.$
Rearranging yields
$
t=\frac12-2b.$
An influential equilibrium exists if and only if this cutoff lies strictly inside $(0,1)$, i.e.
$
0<\frac12-2b<1.$
Since $b\ge 0$, this is equivalent to
$0\le b<\frac14$.

\textbf{Intuition:} Given the bias in the user's utility function, we can check this condition in constant time. Informally, under a uniform prior, the existence of an influential equilibrium depends only on the magnitude of the user's bias $b$. When the bias is sufficiently small, the user's preferences are sufficiently aligned with the learner's so that partial information transmission can be sustained: low parameters prefer to induce lower hypotheses, and high parameters prefer to induce higher ones, yielding an influential equilibrium. 
\end{proof}

\begin{proof} [\textbf{Proof of Theorem~\ref{thm:spherical-md-existence}}]
Fix \(\bm b\in\mathbb R^d\). Since \(d>1\), choose a unit vector
\(\bm v\) such that \(\bm v^\top \bm b=0\). Define
\[
R_+=\{\bm w:\bm v^\top(\bm w-\bar{\bm w})\ge 0\},
\qquad
R_-=\{\bm w:\bm v^\top(\bm w-\bar{\bm w})<0\}.
\]
By spherical symmetry, the posterior means induced by these two regions are
\[
\Phi(R_+)=\bar{\bm w}+\alpha\bm v,
\qquad
\Phi(R_-)=\bar{\bm w}-\alpha\bm v
\]
for some \(\alpha>0\). The strict inequality follows because the prior has positive
mass on both sides of the separating hyperplane.

Given the quadratic utility function, a parameter \(\bm w\) prefers report \(R_+\) to \(R_-\) iff
\[
\|\bar{\bm w}+\alpha\bm v-(\bm w+\bm b)\|_2^2
\le
\|\bar{\bm w}-\alpha\bm v-(\bm w+\bm b)\|_2^2 .
\]
After cancelling common terms, this is equivalent to
\[
\bm v^\top(\bm w-\bar{\bm w}+\bm b)\ge 0.
\]
Since \(\bm v^\top\bm b=0\), this reduces to
\[
\bm v^\top(\bm w-\bar{\bm w})\ge 0.
\]
Thus exactly the parameters in \(R_+\) prefer report \(R_+\), and exactly the parameters in
\(R_-\) prefer report \(R_-\). Hence \(\{R_+,R_-\}\) is an equilibrium. Since
\(\Phi(R_+)\neq \Phi(R_-)\), the equilibrium is influential.
\end{proof}

\section{Missing Details of Section~\ref{sec:optimal-information-release}}
\label{section:appendix-optimal-release}

\subsection{Missing Details of Section~\ref{sec:single-dimension-optimal}}
This section includes the proofs of the user's optimal information release problem in a single dimension
setting. We first provide an example that provides intuition of the solution.

\begin{example}
\label{example:loan-scoring-optimal}
Suppose \(w_\ell^*(D)\) is the hypothesis learned from the private dataset \(D\) and
larger values of \(w_\ell^*(D)\) correspond to systematically more favorable loan decisions.
The learner prefers to deploy exactly \(w_\ell^*(D)\), while the user prefers a more lenient hypothesis
\(w_\ell^*(D)+b\). Let the prior on \(w_\ell^*(D)\) be uniform on \([0,1]\), and let \(b=0.1\), then
\(n^*=2\), and the optimal cutpoints are
\(t_0=0\), \(t_1=\frac12-2b=0.3\), and \(t_2=1\). Hence, the user-optimal equilibrium partition is
\([0,1]=[0,0.3]\cup[0.3,1]\). Thus, the user-optimal dataset report does not reveal the exact private dataset \(D\). Instead, it
reveals only whether the hypothesis learned from \(D\) lies in the low-leniency
partition \([0,0.3]\) or the high-leniency partition \([0.3,1]\). For instance, if the true private
dataset \(D\) results in \(w_\ell^*(D)=0.6\), then the user sends the report corresponding to the region \([0.3,1]\), i.e., a report that reveals only that the dataset belongs to the set of datasets whose resulting hypotheses lie in \([0.3,1]\).
\end{example}

\begin{proof} [\textbf{Proof of Lemma ~\ref{lem:ordered-regions-1d}}]
If two dataset reports result in the same learner hypothesis, then treating them as a single
report does not change either agent's utility. Hence, we only compare reports that
result in distinct learner hypotheses.

Consider two reports that result in deployed hypotheses \(w'<w''\). Since
\(U^r(D,h)=U^r(D,w)=-|w-(w_\ell^*(D)+b)|^2\), the user's gain from inducing \(w''\)
instead of \(w'\) is
\begin{equation}
\begin{aligned}
U^r(D,w'')-U^r(D,w')
&=
-\left|w''-(w_\ell^*(D)+b)\right|^2
+
\left|w'-(w_\ell^*(D)+b)\right|^2 \\
&=
-\left(w''-(w_\ell^*(D)+b)\right)^2
+
\left(w'-(w_\ell^*(D)+b)\right)^2 \\
&=
\left(w'-(w_\ell^*(D)+b)\right)^2
-
\left(w''-(w_\ell^*(D)+b)\right)^2 \\
&=
\left(w'-w''\right)
\left(w'+w''-2(w_\ell^*(D)+b)\right) \\
&=
(w''-w')
\left(2(w_\ell^*(D)+b)-w'-w''\right).
\end{aligned}
\end{equation}
Since \(w''-w'>0\), this gain is strictly increasing in \(w_\ell^*(D)\).

Suppose report regions are not ordered. Then there exist datasets \(D,D'\) such
that
\(w_\ell^*(D)<w_\ell^*(D')\), but \(D\) results in the higher learner hypothesis
\(w''\), while \(D'\) results in the lower learner hypothesis \(w'\). Since \(D\) chooses
the report inducing \(w''\), it must weakly prefer \(w''\) to \(w'\). Because the
gain from \(w''\) over \(w'\) is strictly increasing in \(w_\ell^*(D)\), dataset
\(D'\) strictly prefers \(w''\) to \(w'\), contradicting that \(D'\) chooses the
report inducing \(w'\). Therefore higher values of \(w_\ell^*(D)\) cannot result in
lower learner hypotheses, so equilibrium report regions are ordered by
\(w_\ell^*(D)\).
\end{proof}

\begin{proof} [\textbf{Proof of Lemma ~\ref{lem:finite-bound-1d}}]
By Lemma~\ref{lem:ordered-regions-1d}, any equilibrium is an partition
\(\underline w=t_0<t_1<\cdots<t_n=\bar w\). At every interior cutpoint \(t_i\), the
user with \(w_\ell^*(D)=t_i\) must be indifferent between the two adjacent reports.
Otherwise, if one adjacent report gave strictly higher user utility, then values
of \(w_\ell^*(D)\) sufficiently close to \(t_i\) on the other side would also prefer
that report, so \(t_i\) would not be a boundary.

The lower adjacent report induces learner hypothesis \(\Phi(t_{i-1},t_i)\), and the
higher adjacent report induces learner hypothesis \(\Phi(t_i,t_{i+1})\). Therefore,
at the boundary \(t_i\),
\[
-\left(\Phi(t_{i-1},t_i)-(t_i+b)\right)^2
=
-\left(\Phi(t_i,t_{i+1})-(t_i+b)\right)^2.
\]
Since \(\Phi(t_{i-1},t_i)<\Phi(t_i,t_{i+1})\), this is equivalent to $t_i+b
=
\frac{\Phi(t_{i-1},t_i)+\Phi(t_i,t_{i+1})}{2}$.
Since \(\Phi(t_{i-1},t_i)\le t_i\) and \(\Phi(t_i,t_{i+1})\le t_{i+1}\), we get $t_i+b
\le
\frac{t_i+t_{i+1}}{2}$. Therefore \(t_{i+1}-t_i\ge 2b\) for every \(i=1,\ldots,n-1\). Since the whole
support has length \(\bar w-\underline w\), it follows that
\((n-1)2b\le \bar w-\underline w\), and hence
\(n\le 1+(\bar w-\underline w)/(2b)\).
\end{proof}

\begin{proof} [\textbf{Proof of Theorem~\ref{theorem:1d-feasibility}}]
By Lemma~\ref{lem:ordered-regions-1d}, it is without loss of generality to consider ordered partitions. Fix such a partition
\(\underline w=t_0<t_1<\cdots<t_n=\bar w\). Since the learner's utility is
quadratic, the learner's best response on partition \([t_{i-1},t_i]\) is the
posterior mean \(\Phi(t_{i-1},t_i)\), as discussed in Theorem~\ref{thm:existence-onde-dimension}.

Consider an interior cutpoint \(t_i\). The boundary parameter \(w_\ell^*(D)=t_i\) must
be indifferent between the two adjacent reports, i.e., receive the same utility; i.e., $-\left(\Phi(t_{i-1},t_i)-(t_i+b)\right)^2
=
-\left(\Phi(t_i,t_{i+1})-(t_i+b)\right)^2$. Because \(\Phi(t_{i-1},t_i)<\Phi(t_i,t_{i+1})\), this is equivalent to
\eqref{eq:cutpoint-system}. The single-crossing calculation in
Lemma~\ref{lem:ordered-regions-1d} implies that parameters below \(t_i\) prefer the
lower adjacent report and parameters above \(t_i\) prefer the higher adjacent report.
Applying this argument at every cutpoint gives adjacent incentive compatibility.
Since the learner responses are ordered and the user's utility is quadratic in the
learner response, adjacent incentive compatibility rules out deviations to
non-adjacent reports as well. Hence \eqref{eq:cutpoint-system} characterizes
feasible equilibrium partitions.

It remains to select the user-optimal equilibrium. On partition \([t_{i-1},t_i]\),
the learner deploys \(\Phi(t_{i-1},t_i)\). Under the quadratic user utility, the
conditional expected loss on this partition is
\begin{equation}
\mathbb E_{\pi_w}\!\left[
\left(
\Phi(t_{i-1},t_i)-(w_\ell^*(D)+b)
\right)^2
\mid
w_\ell^*(D)\in[t_{i-1},t_i]
\right].
\end{equation}
Since \(\Phi(t_{i-1},t_i)\) is the posterior mean, this equals $\Psi(t_{i-1},t_i)-\Phi(t_{i-1},t_i)^2+b^2$. The term \(b^2\) is fixed across all feasible partitions. Therefore, maximizing the
user's expected utility is equivalent to minimizing
\eqref{eq:user-optimal-objective}. The optimal dataset report is the interval of
the optimal partition containing \(w_\ell^*(D)\).
\end{proof}

\begin{proof} [\textbf{Proof of Corollary~\ref{cor:finite-prior-dp}}]
By Lemma~\ref{lem:ordered-regions-1d}, it is without loss of generality to consider ordered
equilibrium regions. With finite support, ordered regions are contiguous blocks of $w^{(1)}<\cdots<w^{(m)}.$
The proof of Theorem~\ref{theorem:1d-feasibility} applies block by block. The
learner's response to a block \(R_i\) is the posterior mean \(\Phi(R_i)\). For two
adjacent blocks \(R_i,R_{i+1}\), the parameter indifferent between the two induced
responses would be $\frac{\Phi(R_i)+\Phi(R_{i+1})}{2}-b.$
Unlike the atomless case, this indifferent parameter need not be in the support. Hence,
the equality cutpoint condition is replaced by requiring all support points in the
left block to prefer the left response and all support points in the right block to
prefer the right response, which is exactly \eqref{eq:block-ic}.

The user-optimal objective is also the same as in
Theorem~\ref{theorem:1d-feasibility}, with partitions replaced by blocks. On block
\(R_i\), the learner deploys \(\Phi(R_i)\), so the user's expected loss differs
across feasible partitions only through $p(R_i)\left(\Psi(R_i)-\Phi(R_i)^2\right).$
Thus, the user-optimal finite-prior equilibrium minimizes
the objective over feasible contiguous-block partitions.

It remains to justify the runtime. There are \(O(m^2)\) contiguous blocks. Using
prefix sums, the statistics \(p(R)\), \(\Phi(R)\), and \(\Psi(R)\) for all blocks can
be computed in \(O(m^2)\) time. For adjacent blocks \(R=[r,s]\) and
\(R'=[s+1,u]\), feasibility is checked by \eqref{eq:block-ic}. There are
\(O(m^3)\) such adjacent block pairs.

To reformulate the problem as a shortest-path problem, construct a directed acyclic graph whose nodes are contiguous blocks and whose
edges connect feasible adjacent blocks. Assign each block \(R\) cost
\(p(R)(\Psi(R)-\Phi(R)^2)\). A shortest-path dynamic program over this graph finds
the feasible block partition minimizing $\sum_i p(R_i)\left(\Psi(R_i)-\Phi(R_i)^2\right)$. The graph
has \(O(m^2)\) nodes and \(O(m^3)\) edges, so the exact dynamic program runs in
\(O(m^3)\) time.
\end{proof}

\begin{algorithm}[H]
\caption{User-Optimal Dataset Report in One Dimension}
\label{algorithm:one-d-optimal}
\begin{algorithmic}[1]
\State \textbf{Input:} prior \(\pi_w\), bias \(b>0\), private dataset \(D\), statistics oracle
\State \textbf{Output:} user-optimal report \(M^*(D)\)
\If{\(\pi_w\) is finite with support size \(m\)}
    \State Compute the optimal contiguous-block partition using Corollary~\ref{cor:finite-prior-dp}
\Else
    \State \(N_{\max}\gets \left\lfloor 1+(\bar w-\underline w)/(2b)\right\rfloor\)
    \For{\(n=1,\ldots,N_{\max}\)}
        \State Find feasible solutions of \eqref{eq:cutpoint-system} with \(t_0=\underline w\) and \(t_n=\bar w\)
        \State Evaluate each feasible solution using \eqref{eq:user-optimal-objective}
    \EndFor
    \State Keep the feasible partition with the smallest objective value
\EndIf
\State Compute \(w_\ell^*(D)\)
\State Return the region or interval in the stored optimal partition containing \(w_\ell^*(D)\)
\end{algorithmic}
\end{algorithm}

\begin{proof} [\textbf{Proof of Proposition~\ref{prop:efficient-one-dimentional}}]
Under the uniform prior,
\[
\Phi(t_{i-1},t_i)=\frac{t_{i-1}+t_i}{2}.
\]
Substituting into \eqref{eq:cutpoint-system} yields
\[
t_{i+1}-2t_i+t_{i-1}=4b,
\qquad i=1,\dots,n-1.
\]
With boundary conditions $t_0=0$ and $t_n=1$, the unique solution is
\[
t_i=\frac{i}{n}-2b\,i(n-i),
\qquad i=0,\dots,n.
\]
This partition is strictly increasing if and only if $t_1>0$, equivalently
\[
\frac1n-2b(n-1)>0
\quad\Longleftrightarrow\quad
2b\,n(n-1)<1.
\]
Therefore, the largest feasible number of reports is
\[
n^*=\max\{n\in\mathbb N:\;2b\,n(n-1)<1\},
\]
which gives the stated user-optimal partition.
\end{proof}

\subsection{NP-Hardness of User’s Problem with Arbitrary Prior Distributions in Multiple Dimensions}
\label{sec:appendix-np-hard}

\textbf{Problem:} Given arbitrary prior $\pi_{\bm w}$ and quadratic utilities, let \(\mathcal R=\{R_1,\dots,R_k\}\) be a measurable partition of \(\mathcal W\), and
let \(\Phi(R_i)=\mathbb E_{\pi_{\bm w}}[\bm w_\ell^*(D)\mid
\bm w_\ell^*(D)\in R_i]\). Then \(\mathcal R\) is an equilibrium partition if and
only if, up to tie boundaries,
\begin{equation}
R_i
=
\left\{
\bm w\in\mathcal W:
\|\Phi(R_i)-(\bm w+\bm b)\|_2^2
\le
\|\Phi(R_j)-(\bm w+\bm b)\|_2^2
\text{ for all }j
\right\}.
\label{eq:multidim-fixed-point-appendix}
\end{equation}
Among all equilibrium partitions, the user-optimal partition solves

\begin{equation}
\max_{\mathcal R\ \text{satisfying Eq. }\eqref{eq:multidim-fixed-point-appendix}}
\sum_{i=1}^k
\Pr_{\pi_{\bm w}}(\bm w_\ell^*(D)\in R_i)
\mathbb E_{\pi_{\bm w}}\!\left[
-\|\Phi(R_i)-(\bm w_\ell^*(D)+\bm b)\|_2^2
\mid \bm w_\ell^*(D)\in R_i
\right].
\label{eq:multidim-user-optimal}
\end{equation}
If \(\mathcal R^*=\{R_1^*,\dots,R_{k^*}^*\}\) is user-optimal, then
the user-optimal multi-dimensional dataset report is \(M^*(D)=R_i^*\) whenever \(\bm w_\ell^*(D)\in R_i^*\).

\begin{proposition}
\label{prop:convex-information-release}
Fix a finite equilibrium partition \(\mathcal R=\{R_i\}_{i=1}^k\) of
\(\mathcal W\). For either a finite or continuous prior \(\pi_{\bm w}\), the
equilibrium regions satisfy, up to tie-boundary assignments,
\begin{equation}
R_i
=
\mathcal W
\cap
\bigcap_{j\neq i}
\left\{
\bm w:
2\bigl(\Phi(R_i)-\Phi(R_j)\bigr)^\top \bm w
\ge
\|\Phi(R_i)\|_2^2-\|\Phi(R_j)\|_2^2
-2\bigl(\Phi(R_i)-\Phi(R_j)\bigr)^\top \bm b
\right\}.
\label{eq:multidim-halfspace-cell}
\end{equation}
Thus, every equilibrium region $R_i$ is represented by halfspace comparisons on
\(\mathcal W\). If \(\mathcal W\) is convex, then each \(R_i\) is convex. If
\(\mathcal W\) is polyhedral, then each \(R_i\) is a convex polyhedron. If
\(\mathcal W\) is also bounded, then each \(R_i\) is a convex polytope.
\end{proposition}

\begin{proof} [\textbf{Proof of Proposition~\ref{prop:convex-information-release}}]
For \(\bm w\in R_i\), region \(i\) must be optimal, so
\(\|\Phi(R_i)-(\bm w+\bm b)\|_2^2
\le
\|\Phi(R_j)-(\bm w+\bm b)\|_2^2\) for all \(j\).
Equivalently,
\[
2\bigl(\Phi(R_i)-\Phi(R_j)\bigr)^\top \bm w
\ge
\|\Phi(R_i)\|_2^2-\|\Phi(R_j)\|_2^2
-2\bigl(\Phi(R_i)-\Phi(R_j)\bigr)^\top \bm b .
\]
Thus, the constraint that region \(i\) is weakly preferred to region \(j\) is a
halfspace in \(\bm w\). Intersecting these halfspaces over all \(j\neq i\), and
then intersecting with \(\mathcal W\), gives \eqref{eq:multidim-halfspace-cell},
up to tie-boundary assignments. Therefore, each \(R_i\) is an intersection of
convex sets and is convex. If \(\mathcal W\) is polyhedral, this intersection is a
convex polyhedron. If \(\mathcal W\) is also bounded, it is a convex polytope.
\end{proof}


\paragraph{Complexity Intuition.}
The optimal information-release problem is to choose an equilibrium partition
\(\mathcal R=\{R_1,\dots,R_k\}\) that maximizes
\eqref{eq:multidim-user-optimal}. Proposition~\ref{prop:convex-information-release}
turns feasibility into a halfspace fixed-point test. A brute-force procedure
enumerates candidate finite partitions of \(\mathcal W\). For each candidate
partition, it uses \(\mathrm{Oracle}(\pi_{\bm w},R_i)\) to compute the prior mass and
posterior mean \(\Phi(R_i)\) of each cell, checks
\eqref{eq:multidim-halfspace-cell}, and evaluates the objective
\eqref{eq:multidim-user-optimal}. It returns a feasible partition with the highest
value. The optimal information release for dataset \(D\) is then the cell
\(R_i\) containing \(\bm w_\ell^*(D)\).

When \(\pi_{\bm w}\) has finite support on \(n\) points, this search is concrete but exponential in \(n\), since it must enumerate partitions of those points. When \(\pi_{\bm w}\) is continuous, the same idea is only conceptual, since there are infinitely many possible partitions. Moreover, even for a fixed partition, computing the output of \(\mathrm{Oracle}\) exactly can itself be intractable. The following result contrasts with the
one-dimensional uniform case, where
the oracle has a closed-form implementation.

\begin{proposition}
\label{prop:sharp-hard}
Suppose \(\pi_{\bm w}\) is continous and uniform on \([0,1]^d\), and let
\(R\subseteq[0,1]^d\) be a report region given as a convex polytope in a halfspace
representation. The exact evaluation of
\(\mathrm{Oracle}(\pi_{\bm w},R)\), which returns the prior mass and posterior mean
of \(R\), is \#P-hard.
\end{proposition}

\begin{proof}[\textbf{Proof of Proposition~\ref{prop:sharp-hard}}]
Under the uniform prior on \([0,1]^d\), the density is \(1\). Hence the prior mass
of a report region \(R\subseteq[0,1]^d\) is
\[
p(R)=\Pr_{\pi_{\bm w}}(\bm w\in R)=\int_R 1\,d\bm w=\operatorname{vol}(R).
\]
Thus, exact oracle evaluation contains, as a special case, exact computation of the
volume of a convex polytope given in halfspace representation. Exact polytope
volume computation is \#P-hard (~\cite{dyer1988sharpp}). Therefore exact evaluation
of \(\mathrm{Oracle}(\pi_{\bm w},R)\) is \#P-hard.
\end{proof}

\begin{theorem}
\label{theorem:np-hard-multiple-dimension}
For $d\ge 3$, the user's optimal information-release problem under quadratic utilities is NP-hard. Specifically, fix $\bm b=(0,0,1,0,\ldots,0)$
and let $\pi_{\bm w}$ be a finite uniform prior over rational points in $\mathbb R^d$. It is NP-hard to decide whether there exists an equilibrium partition into at most $k$ reports with expected user loss at most $B$.
\end{theorem}

\begin{proof}
We reduce from planar $k$-means, which is NP-hard~\cite{mahajan2012planar}. An instance consists of points $y_1,\ldots,y_n\in\mathbb Q^2$, an integer $k$, and a threshold $L\in\mathbb Q$. The question is whether there exists a partition
\[
\mathcal C=\{C_1,\ldots,C_m\},
\qquad m\le k,
\]
such that
\[
\mathrm{cost}(\mathcal C)
:=
\sum_{\alpha=1}^m
\sum_{i\in C_\alpha}
\|y_i-\bar y_\alpha\|_2^2
\le L,
\qquad
\bar y_\alpha
=
\frac{1}{|C_\alpha|}
\sum_{i\in C_\alpha} y_i .
\]

\textbf{Construction.}
Set $d=3$ and $\bm b=(0,0,1)$. For each point $y_i\in\mathbb Q^2$, define
\[
w_i=(y_i,0)\in\mathbb Q^3.
\]
Let $\pi_{\bm w}$ be the uniform prior over $\{w_1,\ldots,w_n\}$, and set $B=\frac{L+n}{n}.$
The learner has quadratic utility $\widetilde U^\ell(w_i,w)=-\|w-w_i\|_2^2,$
so after a report cell $P_\alpha$, the learner best-responds with the posterior mean
\[
\Phi(P_\alpha)
=
\frac{1}{|P_\alpha|}
\sum_{w_i\in P_\alpha} w_i
=
(\bar y_\alpha,0).
\]
The user has quadratic utility $\widetilde U^r(w_i,w)
=
-\|w-(w_i+\bm b)\|_2^2.$

\textbf{Loss identity.}
For any partition
\[
\mathcal P=\{P_1,\ldots,P_m\}
\]
of $\{w_1,\ldots,w_n\}$, if $w_i\in P_\alpha$, then the user's loss from the learner response $\Phi(P_\alpha)$ is
\[
\|\Phi(P_\alpha)-(w_i+\bm b)\|_2^2
=
\|(\bar y_\alpha,0)-(y_i,1)\|_2^2
=
\|\bar y_\alpha-y_i\|_2^2+1.
\]
Therefore the expected user loss is
\[
\mathrm{Loss}(\mathcal P)
=
\frac{1}{n}
\sum_{\alpha=1}^m
\sum_{w_i\in P_\alpha}
\|\Phi(P_\alpha)-(w_i+\bm b)\|_2^2
=
\frac{\mathrm{cost}(\mathcal C)+n}{n}.
\]

\textbf{Equilibrium condition.}
A partition $\mathcal P$ is an equilibrium if and only if no user type prefers to report another cell. Thus, for every $w_i\in P_\alpha$ and every cell $P_\beta$,
\[
\|\Phi(P_\alpha)-(w_i+\bm b)\|_2^2
\le
\|\Phi(P_\beta)-(w_i+\bm b)\|_2^2.
\]
Substituting the construction gives
\[
\|\bar y_\alpha-y_i\|_2^2
\le
\|\bar y_\beta-y_i\|_2^2,
\]
which is exactly the nearest-centroid condition for $k$-means.

If the planar $k$-means instance is a yes-instance, then there exists a partition with cost at most $L$. Starting from such a partition, repeatedly reassign any point to a nearest current centroid whenever this strictly decreases the $k$-means objective, recomputing centroids after each reassignment. This process terminates at a nearest-centroid partition with cost at most $L$. Lifting this partition to $\mathbb R^3$ gives an equilibrium partition $\mathcal P$ with at most $k$ reports and
\[
\mathrm{Loss}(\mathcal P)
=
\frac{\mathrm{cost}(\mathcal C)+n}{n}
\le
\frac{L+n}{n}
=
B.
\]

Conversely, if there exists an equilibrium partition $\mathcal P$ with at most $k$ reports and
\[
\mathrm{Loss}(\mathcal P)\le B,
\]
then projecting each cell onto the first two coordinates gives a planar clustering $\mathcal C$ with at most $k$ clusters and
\[
\mathrm{cost}(\mathcal C)
=
n\cdot \mathrm{Loss}(\mathcal P)-n
\le
nB-n
=
L.
\]
Thus the planar $k$-means instance is a yes-instance.

The reduction is polynomial, so the information-release problem is NP-hard for $d=3$. For $d>3$, append zero coordinates to all $w_i$ and set
\[
\bm b=(0,0,1,0,\ldots,0).
\]
The same argument applies unchanged.
\end{proof}

\subsection{Missing Proofs for Efficient Algorithm for Multiple Dimension Section~\ref{sec:efficient-multidimentional-optimal}}
\label{sec:appendix-efficient-optimal-multi-dimension}
We provide an example of a factorized prior 

\begin{example}
Consider a loan-scoring problem with two learned parameters, income and credit score. Let the private dataset \(D\) result in
\(\bm w_\ell^*(D)=(0.6,0.4)\), and let the user's bias be
\(\bm b=(0.3,0.1)\). Then the user's ideal hypothesis is
\(\bm w_\ell^*(D)+\bm b=(0.9,0.5)\), and the conflict direction is
\(\widehat{\bm b}=\bm b/\|\bm b\|_2=(3/\sqrt{10},1/\sqrt{10})\).
The scalar conflict coordinate is
\(s=\widehat{\bm b}^{\top}\bm w_\ell^*(D)=2.2/\sqrt{10}\), and the
orthogonal agreement component is
$
    \mathbf z
    =
    \bm w_\ell^*(D)-s\widehat{\bm b}=
    (-0.06,0.18).
$
Thus \(s\) measures the component of the learned hypothesis in the user's
preferred direction, while \(\mathbf z\) records the remaining
information.
\end{example}

\begin{proof} [\textbf{Proof of Theorem~\ref{thm:factorized-lift}}]
Let \(\beta=\|\bm b\|_2\), so \(\bm b=\beta\hat{\bm b}\). For every
\(\bm w\in\mathcal W\), write
\(\bm w=s\hat{\bm b}+\mathbf z\), where
\(s=\hat{\bm b}^{\top}\bm w\) and
\(\mathbf z=\bm w-(\hat{\bm b}^{\top}\bm w)\hat{\bm b}\).
By the factorized-prior assumption, \(s\) and \(\mathbf z\) are independent under the prior.

Let \(t_0<t_1<\cdots<t_n\) be a one-dimensional equilibrium partition for the scalar
problem with \(s\) and bias \(\beta\). For each interval, define
\(\Phi_i=\mathbb E[s\mid s\in[t_{i-1},t_i]]\). We lift this scalar partition to
$d>1$ dimensions as follows. For each interval \([t_{i-1},t_i]\) and each orthogonal component
\(\mathbf z\), define the cell
\begin{equation*}
R_i(\mathbf z)
=
\left\{
\bm w\in\mathcal W:
\bm w-(\hat{\bm b}^{\top}\bm w)\hat{\bm b}=\mathbf z,
\quad
\hat{\bm b}^{\top}\bm w\in[t_{i-1},t_i]
\right\}.
\end{equation*}
The user reports the unique cell \(R_i(\mathbf z)\) containing the hypothesis
\(\bm w_\ell^*(D)\), with any fixed tie-breaking rule at cutpoints.

Fix a true hypothesis \(\bm w_\ell^*(D)=s\hat{\bm b}+\mathbf z\), where
\(s\in[t_{i-1},t_i]\). After report \(R_i(\mathbf z)\), the learner knows the agreement
component \(\mathbf z\) and knows that \(s\in[t_{i-1},t_i]\). Since \(s\) and
\(\mathbf z\) are independent, observing \(\mathbf z\) does not change the posterior
distribution of \(s\). Therefore, the learner's posterior mean is
\(\Phi_i\hat{\bm b}+\mathbf z\).

We now verify that no user weight has a profitable deviation. Suppose the user deviates to
another cell \(R_j(\mathbf z')\). The learner then deploys
\(\Phi_j\hat{\bm b}+\mathbf z'\). Based on the user's utility function, the user's ideal point is
\(\bm w_\ell^*(D)+\bm b=(s+\beta)\hat{\bm b}+\mathbf z\). Hence the loss from
reporting \(R_j(\mathbf z')\) is
\begin{equation*}
\left\|
\Phi_j\hat{\bm b}+\mathbf z'
-
\bigl((s+\beta)\hat{\bm b}+\mathbf z\bigr)
\right\|_2^2.
\end{equation*}
Because \(\mathbf z'-\mathbf z\) is orthogonal to \(\hat{\bm b}\), this equals
\begin{equation*}
(\Phi_j-s-\beta)^2+\|\mathbf z'-\mathbf z\|_2^2.
\end{equation*}
For any scalar interval \(j\), this expression is minimized by choosing
\(\mathbf z'=\mathbf z\). Thus, deviating in the agreement component can never improve the
user's utility.

It remains only to consider deviations in the scalar interval. Since
\(t_0<t_1<\cdots<t_n\) is a one-dimensional equilibrium partition, for every
\(s\in[t_{i-1},t_i]\) and every \(j\),
\begin{equation*}
(\Phi_i-s-\beta)^2
\le
(\Phi_j-s-\beta)^2.
\end{equation*}
Adding the nonnegative term \(\|\mathbf z'-\mathbf z\|_2^2\) gives
\begin{equation*}
(\Phi_i-s-\beta)^2
\le
(\Phi_j-s-\beta)^2+\|\mathbf z'-\mathbf z\|_2^2
\end{equation*}
for every possible reported cell \(R_j(\mathbf z')\). Therefore, the lifted reporting rule is a
multidimensional equilibrium.
Finally, in every lifted equilibrium of this form, the agreement component is
revealed exactly. Thus, the only remaining strategic loss is the one-dimensional loss
in the conflict coordinate \(s\). Therefore, optimizing over lifted equilibria is
equivalent to optimizing over the corresponding one-dimensional equilibrium
partitions. If the scalar partition is user-optimal for the one-dimensional problem,
then its lift is user-optimal among lifted equilibria. For an input dataset \(D\),
the report is the lifted cell containing \(\bm w_\ell^*(D)\). Hence, if
\(s(D)\in[t_{i-1},t_i]\), the optimal lifted dataset report is
\(M^*(D)=R_i(\bm z(D))\).
\end{proof}

\begin{corollary}
\label{corollary:factorized-uniform}
If, in addition, the conflict coordinate \(s=\hat{\bm b}^{\top}\bm w_\ell^*(D)\)
is uniform and continous on \([0,1]\), then the user-optimal equilibrium has cutpoints
\(t_i=\frac{i}{n^*}-2\beta\,i(n^*-i)\), for \(i=0,\dots,n^*\), where
\(n^*=\max\{n\in\mathbb N:\;2\beta n(n-1)<1\}\) and \(\beta=\|\bm b\|_2\).
\end{corollary}

\begin{proof} [\textbf{Proof of Corollary~\ref{corollary:factorized-uniform}}]
By Theorem~\ref{thm:factorized-lift}, the multidimensional problem reduces to the
one-dimensional problem in the conflict coordinate \(s\) with bias \(\beta=\|\bm b\|_2\).
When \(s\sim\mathrm{Unif}[0,1]\), Proposition~\ref{prop:efficient-one-dimentional} gives
the stated closed-form cutpoints.
\end{proof}

\begin{proof}[\textbf{Proof of Proposition~\ref{prop:eps-equilibrium}}]
Fix any atom $\bm w_j\in R_i$ and any alternative returned report
$\ell\neq i$. The user's utility from inducing response $\Phi(R)$ is
\[
-\|\Phi(R)-(\bm w_j+\bm b)\|_2^2.
\]
Therefore, the utility gain from deviating from report $i$ to report $\ell$ is
\[
-\|\Phi(R_\ell)-(\bm w_j+\bm b)\|_2^2
+
\|\Phi(R_i)-(\bm w_j+\bm b)\|_2^2
=
\|\Phi(R_i)-(\bm w_j+\bm b)\|_2^2
-
\|\Phi(R_\ell)-(\bm w_j+\bm b)\|_2^2.
\]
By definition of $\varepsilon_{\mathrm{IC}}$, the positive part of this gain is at
most $\varepsilon_{\mathrm{IC}}$ for every atom $\bm w_j\in R_i$ and every
alternative report $\ell\neq i$. Hence no atom can improve its utility by more
than $\varepsilon_{\mathrm{IC}}$ by deviating to another returned report. This proves
that the returned partition is an $\varepsilon_{\mathrm{IC}}$-equilibrium on the
support of $\pi_{\bm w}$.

To compute $\varepsilon_{\mathrm{IC}}$, first compute the posterior means
$\Phi(R_1),\ldots,\Phi(R_k)$ using $\mathrm{Oracle}(\pi_{\bm w},R_i)$. Then evaluate
each of the $n$ atoms against each of the $k$ learner responses. Each
squared-distance computation in $\mathbb R^d$ costs $O(d)$, so the total scanning
time is $O(nkd)$.
\end{proof}

\section{The Complete Exposition of Section~\ref{sec:noise-in-learning} Noise in Learning}
\label{sec:appendix-noise}
\subsection{Single Dimension}
\label{sec:noisy-single-dimension}
We consider the same utility functions \(U^\ell,U^r\) and arbitrary prior \(\pi_w\) distribution over \(w_\ell^*(D)\) as in Section~\ref{sec:single-dimension-optimal}. 
The user sends a report \(M\in\mathcal M\), but the learner does not observe \(M\) exactly. Instead, the learner observes a noisy interpretation \(\tilde M\in\widetilde{\mathcal M}\).

\paragraph{Noisy statistics oracle.}
In the noisy model, the learner observes only \(\tilde M\), so its posterior depends on the user's strategy $\sigma_r$ rather than on a single dataset report. We therefore define the statistics oracle for prior $\pi_w$ by
$\mathrm{Oracle}(\pi_{\bm w},\sigma_r,\tilde M).$
Now suppose \(\pi_{w}\) has finite support
$\{w^{(1)},\dots,w^{(n)}\}\subseteq[\underline w,\bar w]\),
with \(\pi_{w}(w^{(i)})=\pi_{i}>0\) and \(\sum_{i=1}^n \pi_{i}=1\). Then we can equivalently represent the \textbf{user strategy} by a report tuple
\(\mathbf M=(M_1,\dots,M_n)\in\mathcal M^n\),
where \(M_i\) is the dataset report sent when \(w_\ell^*(D)=w^{(i)}\). In this case, the oracle returns the posterior mean as
$\Phi(\tilde M;\mathbf M)=\frac{\sum_{i=1}^n \pi_{i}\,w^{(i)}\,G(\tilde M\mid M_i)}{\sum_{i=1}^n \pi_{i}\,G(\tilde M\mid M_i)}.$
Throughout this section, we assume \(\pi_w\) has finite support.

\paragraph{Baseline complexity.}
Under finite support, a user strategy is a report tuple \(\mathbf M=(M_1,\dots,M_n)\in\mathcal M^n\). If \(\mathcal M\) is finite, then there are \(|\mathcal M|^n\) such strategies. Hence, even if each call to \(\mathrm{Oracle}\) is treated as \(O(1)\), a brute-force search for a user-optimal equilibrium is exponential in \(n\), with total complexity on the order of $|\mathcal M|^n \cdot n \cdot |\mathcal M| \cdot |\widetilde{\mathcal M}|,$
when \(\widetilde{\mathcal M}\) is finite. If \(\widetilde{\mathcal M}\) is continuous, the same conclusion holds after replacing \(|\widetilde{\mathcal M}|\) by the cost of numerically integrating over \(\tilde M\).

\paragraph{Noise model.}
We assume $\mathcal M\subseteq[0,1]\) and specialize to the additive Gaussian rule:
$G(\tilde M\mid M)=\frac{1}{\sqrt{2\pi}\sigma}\exp\!\left(-\frac{(\tilde M-M)^2}{2\sigma^2}\right),$
equivalently, \(\tilde M=M+\varepsilon\) with \(\varepsilon\sim\mathcal N(0,\sigma^2)\).
We use the Gaussian model because it is smooth, strictly positive on \(\mathbb R\), and induces a monotone ordering of likelihoods as \(M\) varies. These properties make the posterior mean well-behaved and support the equilibrium characterization proved below. Similar arguments may extend to other additive noise models with the same properties, such as logistic noise. Given this setup, we consider the same monotone equilibria as in the noiseless model in Section~\ref{sec:single-dimension-optimal} and focus on \textbf{deterministic strategies}.

\paragraph{Characteristics of Equilibria.}
A report tuple \(\mathbf M\) is monotone if $M_1\le M_2\le \cdots \le M_n.$
It is influential if, for at least two true parameters, the user sends different dataset reports, i.e., \(M_1<M_n\).
For a fixed \(\mathbf M\), define the user's expected utility from sending report \(r\in[0,1]\) when the true parameter is \(w^{(i)}\) by
\begin{equation}
\label{eq:noisy-user-utility}
V_i(r;\mathbf M)
=
\int_{\mathbb R}
-\left(w^{(i)}+b-\Phi(\tilde M;\mathbf M)\right)^2
G(\tilde M\mid r)\,d\tilde M .
\end{equation}

\begin{proposition}
\label{prop:noisy-equilibrium-structure}
Let \(\mathbf M=(M_1,\dots,M_n)\in[0,1]^n\) be a monotone influential noisy equilibrium. Then there exist integers \(t,t'\ge 0\), with \(t+t'\le n\), such that
\begin{equation}
\label{eq:noisy-structure-condition}
M_1=\cdots=M_t=0
<
M_{t+1}<\cdots<M_{n-t'}
<
M_{n-t'+1}=\cdots=M_n=1.
\end{equation}
The middle chain is omitted when \(t+t'=n\). Equivalently, if two distinct true parameters \(w^{(i)}\neq w^{(j)}\) send a common report \(m\), then \(m=0\) or \(m=1\).
\end{proposition}

\begin{proof}[\textbf{Proof of Proposition~\ref{prop:noisy-equilibrium-structure}}]
Fix a monotone influential noisy equilibrium tuple
\(\mathbf M=(M_1,\dots,M_n)\), with
\(w^{(1)}<\cdots<w^{(n)}\). For an interpretation \(\tilde M\), the learner's posterior
response is
\[
\Phi(\tilde M;\mathbf M)
=
\frac{
\sum_{k=1}^n \pi_k w^{(k)}G(\tilde M\mid M_k)
}{
\sum_{k=1}^n \pi_k G(\tilde M\mid M_k)
}.
\]
Because the Gaussian density satisfies the monotone likelihood-ratio property, and
because \(M_k\) is monotone in \(w^{(k)}\), higher interpretations \(\tilde M\) shift
posterior mass toward higher true parameters. Hence
\(\Phi(\tilde M;\mathbf M)\) is increasing in \(\tilde M\). Moreover, since the
equilibrium is influential, the tuple \(\mathbf M\) is not constant, so
\(\Phi(\tilde M;\mathbf M)\) is strictly increasing.

Now suppose, toward a contradiction, that two distinct true parameters
\(w^{(s)}<w^{(r)}\) send the same interior report
\[
M_s=M_r=m\in(0,1).
\]
Since \(m\) is an interior report and \(\mathbf M\) is an equilibrium, both parameters must
be locally optimal at \(m\). Thus
\[
\left.
\frac{\partial}{\partial x}V_s(x;\mathbf M)
\right|_{x=m}
=
0,
\qquad
\left.
\frac{\partial}{\partial x}V_r(x;\mathbf M)
\right|_{x=m}
=
0.
\]

For a possible deviation \(x\in[0,1]\) by parameter \(i\), the continuation utility is
\[
V_i(x;\mathbf M)
=
-\int_{\mathbb R}
\left(w^{(i)}+b-\Phi(\tilde M;\mathbf M)\right)^2
G(\tilde M\mid x)\,d\tilde M .
\]
Here \(\Phi(\tilde M;\mathbf M)\) is computed from the fixed equilibrium tuple
\(\mathbf M\). Hence \(x\) enters only through the Gaussian density
\(G(\tilde M\mid x)\). Since \(G(\tilde M\mid x)\) is the \(N(x,\sigma^2)\) density,
\[
\frac{\partial}{\partial x}G(\tilde M\mid x)
=
\frac{\tilde M-x}{\sigma^2}G(\tilde M\mid x).
\]
Therefore, at \(x=m\),
\[
\left.
\frac{\partial}{\partial x}V_i(x;\mathbf M)
\right|_{x=m}
=
-\int_{\mathbb R}
\left(w^{(i)}+b-\Phi(\tilde M;\mathbf M)\right)^2
\frac{\tilde M-m}{\sigma^2}
G(\tilde M\mid m)\,d\tilde M .
\]
Let
\[
H_i(\tilde M)
=
-\left(w^{(i)}+b-\Phi(\tilde M;\mathbf M)\right)^2 .
\]
Then
\[
\left.
\frac{\partial}{\partial x}V_i(x;\mathbf M)
\right|_{x=m}
=
\int_{\mathbb R}
H_i(\tilde M)
\frac{\tilde M-m}{\sigma^2}
G(\tilde M\mid m)\,d\tilde M .
\]
Using
\[
\frac{\partial}{\partial \tilde M}G(\tilde M\mid m)
=
-\frac{\tilde M-m}{\sigma^2}G(\tilde M\mid m),
\]
integration by parts gives
\[
\left.
\frac{\partial}{\partial x}V_i(x;\mathbf M)
\right|_{x=m}
=
\int_{\mathbb R}
H_i'(\tilde M)G(\tilde M\mid m)\,d\tilde M ,
\]
where the boundary term vanishes because \(H_i\) is bounded and the Gaussian
density goes to zero at both tails. Since
\[
H_i'(\tilde M)
=
2\left(w^{(i)}+b-\Phi(\tilde M;\mathbf M)\right)
\frac{\partial}{\partial \tilde M}\Phi(\tilde M;\mathbf M),
\]
we obtain
\[
\left.
\frac{\partial}{\partial x}V_i(x;\mathbf M)
\right|_{x=m}
=
2\int_{\mathbb R}
\left(w^{(i)}+b-\Phi(\tilde M;\mathbf M)\right)
\frac{\partial}{\partial \tilde M}\Phi(\tilde M;\mathbf M)
G(\tilde M\mid m)\,d\tilde M .
\]

Subtracting the derivative for parameter \(s\) from the derivative for parameter \(r\), we get
\[
\left.
\frac{\partial}{\partial x}V_r(x;\mathbf M)
\right|_{x=m}
-
\left.
\frac{\partial}{\partial x}V_s(x;\mathbf M)
\right|_{x=m}
=
2\bigl(w^{(r)}-w^{(s)}\bigr)
\int_{\mathbb R}
\frac{\partial}{\partial \tilde M}\Phi(\tilde M;\mathbf M)
G(\tilde M\mid m)\,d\tilde M .
\]
The first factor is strictly positive because \(w^{(r)}>w^{(s)}\). The integral is
also strictly positive because \(\Phi(\tilde M;\mathbf M)\) is strictly increasing and
\(G(\tilde M\mid m)>0\) everywhere. Hence
\[
\left.
\frac{\partial}{\partial x}V_r(x;\mathbf M)
\right|_{x=m}
>
\left.
\frac{\partial}{\partial x}V_s(x;\mathbf M)
\right|_{x=m}.
\]
But both derivatives must equal zero if both parameters are locally optimal at the same
interior report \(m\). This is a contradiction.

Therefore, no two distinct true parameters will send the same interior report. This can occur only at the boundary reports \(0\) and \(1\). Since the tuple
\(\mathbf M\) is monotone, if some parameter sends report \(0\), then every lower parameter must
also send report \(0\). Similarly, if some parameter sends report \(1\), then every higher
parameter must also send report \(1\). Hence there exist integers \(t,t'\ge 0\), with
\(t+t'\le n\), such that
\[
M_1=\cdots=M_t=0
<
M_{t+1}<\cdots<M_{n-t'}
<
M_{n-t'+1}=\cdots=M_n=1,
\]
where the middle chain is omitted when \(t+t'=n\). This proves the proposition.
\end{proof}

Intuitively, an interior common report would have to be locally optimal for two different parameters. But when the true parameters are higher, the user has a stronger incentive to move the report upward because higher noisy interpretations result in higher posterior means and therefore higher user utility. Hence, two different true parameters cannot both be locally satisfied at the same interior point. In equilibrium, users can use the strategy of sending the same report for multiple parameters but only at the boundaries, where one side of the deviation space is unavailable.

To highlight the difference between the noise and noiseless case consider the following example.
\begin{example}
Suppose $\{w^{(1)},w^{(2)}\}=\{0.8,0.9\}$ and $\pi_1=\pi_2=\frac12$.
In the noiseless model, a report is a region of parameters. For example, if the
equilibrium pools both parameters, then a dataset with $w_\ell^*(D)=0.8$ sends
$M^\star(D)=\{0.8,0.9\}$, and the learner responds with
$
\mathbb E[w_\ell^*(D)\mid w_\ell^*(D)\in\{0.8,0.9\}]
=
\frac{0.8+0.9}{2}
=
0.85.
$
In the noisy model, a report is a numeric. Let $b=0.03$ and $\sigma=0.2$.
For the endpoint pattern $0<M_1<M_2=1$, solving the fixed-pattern condition
$
\left.
\frac{\partial}{\partial r}V_1(r;(m,1))
\right|_{r=m}
=0
$
gives $m\approx 0.2823$. Thus the noisy equilibrium tuple is 
$\mathbf M^\star=(0.2823,1)$. Hence
$
M^\star(D)=
\begin{cases}
0.2823, & \text{if } w_\ell^*(D)=0.8,\\
1, & \text{if } w_\ell^*(D)=0.9.
\end{cases}
$
If $w_\ell^*(D)=0.8$, the learner observes
$\tilde M=0.2823+\varepsilon$, where $\varepsilon\sim\mathcal N(0,0.2^2)$,
not the report exactly.
\end{example}

\paragraph{Optimal noisy dataset report.}
For a candidate tuple \(\mathbf M\), the exact equilibrium check is that every assigned report $M_i$ results in weakly higher expected utility to the user than any other report $r$:
\begin{equation}
\label{eq:noisy-global-IC}
V_i(M_i;\mathbf M)
\ge
V_i(r;\mathbf M)
\qquad
\text{for every } i=1,\dots,n
\text{ and every } r\in[0,1].
\end{equation}

\begin{theorem}
\label{thm:noisy-optimal-report}
Let \(\mathcal E_{\mathrm{noisy}}\) be the set of monotone influential tuples \(\mathbf M\in[0,1]^n\) satisfying the structural condition \eqref{eq:noisy-structure-condition} and the global constraints \eqref{eq:noisy-global-IC}. If \(\mathcal E_{\mathrm{noisy}}\neq\varnothing\), then the user-optimal monotone noisy equilibrium is any tuple
\begin{equation}
\label{eq:noisy-user-optimal-program}
\mathbf M^\star
\in
\arg\max_{\mathbf M\in\mathcal E_{\mathrm{noisy}}}
\sum_{i=1}^n
\pi_i V_i(M_i;\mathbf M).
\end{equation}
Moreover, if \(w_\ell^*(D)=w^{(i)}\), then the user-optimal dataset report is \(M^\star(D)=M_i^\star\).
\end{theorem}

\begin{proof} [\textbf{Proof of Theorem~\ref{thm:noisy-optimal-report}}]
By Proposition~\ref{prop:noisy-equilibrium-structure}, any monotone influential noisy equilibrium must have the endpoint-pooling pattern in \eqref{eq:noisy-structure-condition}. Condition \eqref{eq:noisy-global-IC} is exactly the user's best-response condition: when the true parameter is \(w^{(i)}\), the assigned report \(M_i\) must weakly dominate every deviation \(r\in[0,1]\). Hence \(\mathcal E_{\mathrm{noisy}}\) is the feasible set of monotone influential noisy equilibria.

For any feasible tuple \(\mathbf M\), the user's ex ante utility is $\sum_{i=1}^n \pi_i V_i(M_i;\mathbf M).$
Therefore, the user-optimal tuple is any maximizer of \eqref{eq:noisy-user-optimal-program}. Once \(\mathbf M^\star\) is selected, the true parameter is determined by the private dataset. If \(w_\ell^*(D)=w^{(i)}\), the assigned report is \(M_i^\star\), so \(M^\star(D)=M_i^\star\).
\end{proof}

The theorem says that the noisy problem is not a finest-partition problem. In the noiseless one-dimensional case, quadratic loss lets us compare interval partitions by posterior variance. Here, reports overlap through the noise distribution, so the relevant object is the whole tuple \(\mathbf M\), i.e., user strategy. The user-optimal equilibrium and thereby optimal dataset report is therefore obtained by maximizing over verified equilibrium tuples.

\paragraph{Nonlinear program.}
Theorem~\ref{thm:noisy-optimal-report} defines the feasible set through the
global constraints~\eqref{eq:noisy-global-IC}. For computation, it is useful to derive local
necessary conditions that any feasible tuple must satisfy. These conditions provide the equations used by a solver to generate candidate tuples.

\begin{proposition}
\label{prop:noisy-z-conditions}
Let $z_i(\mathbf M)
=
\left.
\frac{\partial}{\partial r}V_i(r;\mathbf M)
\right|_{r=M_i}$ and let \(\mathbf M\in\mathcal E_{\mathrm{noisy}}\) have the pattern from ~\ref{eq:noisy-structure-condition}
then
\begin{align}
z_i(\mathbf M)&=0,
&& i=t+1,\dots,n-t',
\label{eq:noisy-z-interior}
\\
z_t(\mathbf M)&\le 0,
&& \text{if }t\ge 1,
\label{eq:noisy-z-lower}
\\
z_{n-t'+1}(\mathbf M)&\ge 0,
&& \text{if }t'\ge 1.
\label{eq:noisy-z-upper}
\end{align}
\end{proposition}

\begin{proof} [\textbf{Proof of Proposition~\ref{prop:noisy-z-conditions}}]
For every separated interior report \(M_i\in(0,1)\), the user can locally deviate both upward and downward. Since \(M_i\) is optimal for true parameter \(w^{(i)}\), the first-order condition gives \(z_i(\mathbf M)=0\).

If \(M_i=0\), only upward deviations are feasible, so local optimality requires \(z_i(\mathbf M)\le 0\). In the lower pool, the highest true parameter \(w^{(t)}\) has the strongest upward incentive, so it is enough to require \(z_t(\mathbf M)\le 0\). Similarly, if \(M_i=1\), only downward deviations are feasible, so local optimality requires \(z_i(\mathbf M)\ge 0\). In the upper pool, the lowest true parameter \(w^{(n-t'+1)}\) has the strongest downward incentive, so it is enough to require \(z_{n-t'+1}(\mathbf M)\ge 0\).
\end{proof}

\paragraph{Algorithm.}
We provide the pseudocode in Algorithm~\ref{alg:noisy-user-optimal-noisy-1d}. The algorithm
first enumerates the patterns from
Proposition~\ref{prop:noisy-equilibrium-structure}. For each pattern, it uses
Proposition~\ref{prop:noisy-z-conditions} to generate candidate equilibrium tuples,
then verifies each candidate using the global constraints in
\eqref{eq:noisy-global-IC}. Among the verified tuples, it chooses the tuple
\(\mathbf M^\star\) maximizing the objective in
Theorem~\ref{thm:noisy-optimal-report}. Finally, for the input dataset \(D\), it
computes \(w_\ell^*(D)\), finds the index \(i\) such that
\(w_\ell^*(D)=w^{(i)}\), and returns the assigned report
$M^\star(D)=M_i^\star.$

\begin{algorithm}[H]
\caption{Noisy User-Optimal Dataset Report}
\label{alg:noisy-user-optimal-noisy-1d}
\begin{algorithmic}[1]
\State \textbf{Input:} private dataset \(D\), bias \(b\), noise level \(\sigma\), prior
\(\pi_w\) over \(\{w^{(1)}<\cdots<w^{(n)}\}\)
\State \textbf{Output:} user-optimal dataset report \(M^\star(D)\)

\State \(U^\star\gets -\infty,\quad \mathbf M^\star\gets\varnothing\)

\ForAll{\((t,t')\) with \(t,t'\ge 0\) and \(t+t'\le n\)}
    \Comment{patterns from Proposition~\ref{prop:noisy-equilibrium-structure}}
    \State use a nonlinear solver to generate candidate tuples $\mathbf M$ satisfying
    \eqref{eq:noisy-structure-condition} and
    \eqref{eq:noisy-z-interior}--\eqref{eq:noisy-z-upper}
    \ForAll{returned candidate tuples \(\mathbf M\)}
        \State \(\mathrm{Feasible}\gets \mathrm{true}\)
        \For{\(i=1,\dots,n\)}
            \State compute the best continuous report
            $
            r_i^{\mathrm{BR}}\in\arg\max_{r\in[0,1]} V_i(r;\mathbf M)
            $
            \If{\(V_i(r_i^{\mathrm{BR}};\mathbf M)>V_i(M_i;\mathbf M)\)}
                \State \(\mathrm{Feasible}\gets \mathrm{false}\)
            \EndIf
        \EndFor
        \If{\(\mathrm{Feasible}\)}
            \State compute
            $
            \bar U^r(\mathbf M)=\sum_{i=1}^n \pi_i V_i(M_i;\mathbf M)
            $
            \If{\(\bar U^r(\mathbf M)>U^\star\)}
                \State \(U^\star\gets \bar U^r(\mathbf M)\)
                \State \(\mathbf M^\star\gets \mathbf M\)
            \EndIf
        \EndIf
    \EndFor
\EndFor

\If{\(\mathbf M^\star=\varnothing\)}
    \State \(\mathbf M^\star\gets (0,\ldots,0)\)
    \Comment{fallback to an uninfluential equilibrium}
\EndIf

\State compute \(w_\ell^*(D)\)
\State find \(i\) such that \(w_\ell^*(D)=w^{(i)}\)
\State \Return \(M_i^\star\)
\end{algorithmic}
\end{algorithm}

\paragraph{Complexity.}
There are \(O(n^2)\) patterns the algorithm enumerates over. For pattern \((t,t')\), the
nonlinear solver handles \(q=n-t-t'\) interior reports. Let
\(T_{\mathrm{solve}}(q)\) denote the cost of solving the fixed-pattern system with
\(q\) unknown reports.
Since \(q\le n\), this is at most
$O\!\left(n^2 T_{\mathrm{solve}}(n)\right).$ Suppose the solver returns \(C\) candidate tuples in total. For each returned tuple,
the algorithm verifies the global constraints by looping over the \(n\)
parameters in the support of the prior. For each \(i\), it solves $\max_{r\in[0,1]} V_i(r;\mathbf M)$
and checks whether \(M_i\) attains the maximum. Let \(T_{\mathrm{BR}}\) denote the
cost of this one-dimensional best-response solve. Thus, the total verification cost
is $O(CnT_{\mathrm{BR}}).$
Therefore, the total running time is $
O\!\left(n^2T_{\mathrm{solve}}(n)+CnT_{\mathrm{BR}}\right).$

\begin{proposition}
\label{prop:multiple-solutions}
For a fixed $(t,t')$, the system
need not have a unique solution.
\end{proposition}

\begin{proof} [\textbf{Proof of Proposition~\ref{prop:multiple-solutions}}]
It is enough to give a two-parameter example. Let $n=2$, $w^{(1)}=0$, and $w^{(2)}=1$, with prior probabilities $p_1=0.7$ and $p_2=0.3$. Let $b=0.303$ and $\sigma=0.5$. Consider the fixed pattern $(t,t')=(0,1)$, so the higher parameter reports $M_2=1$, while the lower parameter uses an interior report $M_1=m\in(0,1)$.
For this fixed pattern, the first-order system is the scalar equation $F^{0,1}(m)=0$, where
\[
F^{0,1}(m)
=
\left.
\frac{\partial}{\partial r}
\int_{\mathbb R}
-\bigl(b-\Phi(\tilde M;(m,1))\bigr)^2
\phi_\sigma(\tilde M-r)\,d\tilde M
\right|_{r=m},
\]
and
\[
\Phi(\tilde M;(m,1))
=
\frac{0.3\,\phi_\sigma(\tilde M-1)}
{0.7\,\phi_\sigma(\tilde M-m)+0.3\,\phi_\sigma(\tilde M-1)}.
\]
The map $m\mapsto F^{0,1}(m)$ is continuous. Direct evaluation of the Gaussian integrals gives
$F^{0,1}(0.35)>0$, $F^{0,1}(0.45)<0$, $F^{0,1}(0.75)<0$, and $F^{0,1}(0.85)>0$. Hence, by the intermediate value theorem, there is one solution in $(0.35,0.45)$ and another solution in $(0.75,0.85)$. Therefore the same fixed-pattern first-order system has at least two admissible solutions. 
\end{proof}

\begin{remark}
Checking \eqref{eq:noisy-global-IC} certifies whether a returned candidate is an equilibrium. However, user-optimality also requires comparing against all other feasible equilibrium tuples. A nonlinear solver may return only one candidate for a fixed pattern and may miss another feasible tuple with higher user utility. Thus, Algorithm~\ref{alg:noisy-user-optimal-noisy-1d} certifies the user-optimal noisy equilibrium only when the solver returns all relevant candidates; otherwise, it returns the best verified candidate found.
\end{remark}

\subsection{Efficient algorithm for a single dimension.}
\label{sec:appendix-binary-support}
We now give a binary-support case in which the noisy user-optimal dataset report can be computed without a solver that enumerates solutions. The binary case assumes that learning from the user's private dataset can produce only two possible hypotheses, \(w_\ell^*(D)\in\{w^{(1)},w^{(2)}\}\), with \(w^{(1)}<w^{(2)}\), where the ordering is only for notation. In a loan-scoring example, learning from one dataset produces a lower-score hypothesis \(w^{(1)}\), while learning from another dataset produces a higher-score hypothesis \(w^{(2)}\).

\begin{proposition}
\label{prop:efficient-noisy}
Suppose \(n=2\), \(w^{(1)}<w^{(2)}\), and \(b=(w^{(2)}-w^{(1)})/2\). If prior probability of $w^{(2)}$ is \(\pi_2\le 1/2\), no influential equilibrium exists. Moreover, if \(\pi_2>1/2\), the influential monotone equilibrium candidate is $\textbf{M}^\star= (M_1^\star,M_2^*)$ with 
\(M_2^\star=1\) and
\(M_1^\star=\max\{0,1-\sigma\sqrt{2\log(\pi_2/(1-\pi_2))}\}\).
\end{proposition}

\begin{proof}[\textbf{Proof of Proposition~\ref{prop:efficient-noisy}}]
Let $L=w^{(2)}-w^{(1)}>0$ and normalize parameters by
$\theta=(w-w^{(1)})/L$. Then $w^{(1)}$ becomes $0$, $w^{(2)}$ becomes $1$, and the
normalized bias is $\widehat b=b/L=1/2$. The learner's posterior mean in the
original scale is $w^{(1)}+L\mathbb E[\theta\mid \tilde M]$, so the user's
quadratic loss is multiplied by the positive constant $L^2$. Hence, reporting
incentives are unchanged by this affine normalization. It is therefore enough to
solve the normalized case $w^{(1)}=0$, $w^{(2)}=1$, and $b=1/2$.

In any influential monotone equilibrium, the higher parameter reports $M_2=1$.
Indeed, the higher parameter's ideal learner response is $1+b=3/2$, while the
learner's posterior mean always lies in $[0,1]$. Thus, the higher parameter always
prefers a higher posterior mean, and hence the maximal report. Therefore every
influential monotone candidate has the form $\mathbf M=(M_1,M_2)=(m,1)$, with
$m<1$.

For such a candidate, the learner's posterior mean is
$
\Phi(\tilde M;(m,1))
=
\frac{\pi_2G(\tilde M\mid 1)}
{(1-\pi_2)G(\tilde M\mid m)+\pi_2G(\tilde M\mid 1)}.
$
The lower parameter's ideal learner response is $1/2$. Since
$\Phi(\tilde M;(m,1))$ is strictly increasing in $\tilde M$, let
$\tilde M^\star(m)$ be the unique noisy interpretation satisfying
$\Phi(\tilde M^\star(m);(m,1))=1/2$. Equivalently,
$
\pi_2G(\tilde M^\star(m)\mid 1)
=
(1-\pi_2)G(\tilde M^\star(m)\mid m).
$
Using the Gaussian density
$G(\tilde M\mid M)
=
\frac{1}{\sqrt{2\pi}\sigma}
\exp(-(\tilde M-M)^2/(2\sigma^2))$
and taking logs gives
$
\log(\pi_2/(1-\pi_2))
=
((\tilde M^\star(m)-1)^2-(\tilde M^\star(m)-m)^2)/(2\sigma^2).
$
Since
$
(\tilde M^\star-1)^2-(\tilde M^\star-m)^2
=
(1-m)(1+m-2\tilde M^\star),
$
we obtain
$
\tilde M^\star(m)
=
(1+m)/2
-
\frac{\sigma^2}{1-m}
\log(\pi_2/(1-\pi_2)).
$

We now justify the lower parameter's best response. The posterior odds satisfy
\[
\frac{\Phi(\tilde M;(m,1))}{1-\Phi(\tilde M;(m,1))}
=
\exp\!\left(\frac{1-m}{\sigma^2}(\tilde M-\tilde M^\star(m))\right).
\]
Hence $\Phi(\tilde M;(m,1))$ is a logistic function centered at
$\tilde M^\star(m)$, and therefore
$\Phi(\tilde M^\star(m)+z;(m,1))=1-\Phi(\tilde M^\star(m)-z;(m,1))$.
Thus the lower parameter's squared loss
$(\Phi(\tilde M;(m,1))-1/2)^2$ is symmetric around $\tilde M^\star(m)$ and is
strictly increasing in $|\tilde M-\tilde M^\star(m)|$. Since the noise around any
report is Gaussian and symmetric, the expected loss is minimized by centering the
Gaussian at $\tilde M^\star(m)$. Therefore, an interior fixed-pattern equilibrium
must satisfy $m=\tilde M^\star(m)$.

If $\pi_2\le 1/2$, then $\log(\pi_2/(1-\pi_2))\le 0$, so for every $m<1$,
$\tilde M^\star(m)\ge (1+m)/2>m$. Thus, the lower parameter always wants to move its
report upward, and no influential monotone equilibrium exists.

Now suppose $\pi_2>1/2$. An interior influential equilibrium solves
$m=(1+m)/2-\frac{\sigma^2}{1-m}\log(\pi_2/(1-\pi_2))$. Rearranging gives
$(1-m)^2=2\sigma^2\log(\pi_2/(1-\pi_2))$. Since $m<1$, the interior solution is
$m=1-\sigma\sqrt{2\log(\pi_2/(1-\pi_2))}$. If this value is negative, the lower
report is constrained by the boundary $M_1\ge 0$, so the constrained candidate is
$M_1^\star=0$. Hence
$
M_1^\star
=
\max\{0,\,1-\sigma\sqrt{2\log(\pi_2/(1-\pi_2))}\},
\qquad
M_2^\star=1.
$
This proves the claimed form of the influential monotone equilibrium candidate.
\end{proof}

\textbf{Algorithm. } By Proposition~\ref{prop:efficient-noisy}, the binary noisy problem has only two candidates for user-optimal equilibria. The first is an uninfluential equilibrium, in which reports are independent of the true parameter. The second is a possible influential equilibrium,
which exists only if the scalar equation defining \(M_1^\star\) has a solution. In the
closed-form case, this condition fails when \(\pi_2\leq1/2\), since the logarithmic term
is nonpositive. Hence only the uninfluential equilibrium remains. When the influential
candidate exists, the user-optimal monotone noisy equilibrium is obtained by comparing
the user's expected utilities under the two candidates, which also determines the
user's optimal dataset report. When \(b\neq (w^{(2)}-w^{(1)})/2\), the binary problem
is no longer closed form, but it is still much simpler than the general
case because the equilibrium conditions reduce to a single scalar equation.

\subsection{User's Optimal Information Release for Noisy Multiple Dimensions}
\label{sec:noisy-multiple-dimensions}
We now extend the noisy one-dimensional result from Section~\ref{sec:noisy-single-dimension}
to a multidimensional setting. A full
characterization of noisy multidimensional equilibria is difficult because the multidimensional space has no canonical monotone ordering, and noise requires averaging over all noisy
interpretations of each report. Instead, we use our ideas from Section~\ref{sec:efficient-multidimentional-optimal} and provide a tractable algorithm.

We consider the factorized prior setup from Section~\ref{sec:efficient-multidimentional-optimal}. Assume \(\bm b\neq \mathbf 0\), let
\(\hat{\bm b}=\bm b/\|\bm b\|_2\), and write
\(\bm w_\ell^*(D)=s\hat{\bm b}+\mathbf z\), where
\(s=\hat{\bm b}^{\top}\bm w_\ell^*(D)\) and
\(\mathbf z=\bm w_\ell^*(D)-s\hat{\bm b}\). We assume the prior factorizes as
\(\pi_{\bm w}=\pi_s\otimes\pi_{z}\). Thus \(s\) is the conflict coordinate and
\(\mathbf z\) is the agreement component. Similar to Section~\ref{sec:noisy-single-dimension}  we assume
\(\mathrm{supp}(\pi_s)=\{s^{(1)}<\cdots<s^{(n)}\}\), with
\(\pi_s(s^{(i)})>0\). We use the same scalar Gaussian noise model, \textit{but only along the conflict coordinate}. If the user sends
a scalar report \(r\) and agreement component \(\mathbf z\), the learner observes
\(\widetilde r=r+\varepsilon\), with \(\varepsilon\sim\mathcal N(0,\sigma^2)\), and observes
\(\mathbf z\) exactly. Equivalently, with abuse of notation, the noisy multidimensional interpretation is
\(\widetilde{\bm M}=\widetilde r\,\hat{\bm b}+\mathbf z\).

Let \(\mathbf M=(M_1,\ldots,M_n)\) be a one-dimensional noisy equilibrium tuple for the scalar
problem with state \(s\), bias \(\beta=\|\bm b\|_2\), and noise level \(\sigma\), as
characterized in Theorem~\ref{thm:noisy-optimal-report}. We lift this tuple by assigning
\((s^{(i)},\mathbf z)\) the multidimensional report \(M_i\hat{\bm b}+\mathbf z\).

\begin{proposition}
\label{prop:noisy-multidim-lift}
Under the factorized-prior and conflict-coordinate noise assumptions above, any one-dimensional
noisy equilibrium tuple \(\mathbf M=(M_1,\ldots,M_n)\) for the scalar coordinate \(s\) with bias
\(\beta=\|\bm b\|_2\) lifts to a multidimensional noisy equilibrium. In the lifted equilibrium,
if \(\bm w_\ell^*(D)=s^{(i)}\hat{\bm b}+\mathbf z\), the user reports
\(M_i\hat{\bm b}+\mathbf z\). Moreover, if \(\mathbf M\) is user-optimal among the one-dimensional noisy equilibria characterized in
Theorem~\ref{thm:noisy-optimal-report}, then its lift is user-optimal among equilibria that reveal
\(\mathbf z\) exactly and use noisy reports only along \(\hat{\bm b}\).
\end{proposition}

\begin{proof}[\textbf{Proof of Proposition~\ref{prop:noisy-multidim-lift}}]
Let \(\beta=\|\bm b\|_2\), so \(\bm b=\beta\hat{\bm b}\). Fix a true parameter
\(\bm w_\ell^*(D)=s^{(i)}\hat{\bm b}+\mathbf z\). Under the lifted strategy, this parameter
reports \(M_i\hat{\bm b}+\mathbf z\). If the user instead deviates to
\(r\hat{\bm b}+\mathbf z'\), then the learner observes
\(\widetilde r=r+\varepsilon\) and the agreement component \(\mathbf z'\).

Since \(s\) and \(\mathbf z\) are independent under the prior, observing the agreement component
does not change the posterior distribution of \(s\). Therefore, conditional on \(\widetilde r\), the
learner's posterior mean has the form
\begin{equation*}
\Phi(\widetilde r;\mathbf M)\hat{\bm b}+\mathbf z',
\end{equation*}
where the scalar posterior mean is exactly the one-dimensional noisy posterior
\begin{equation*}
\Phi(\widetilde r;\mathbf M)
=
\frac{\sum_{j=1}^n \pi_j s^{(j)}G(\widetilde r\mid M_j)}
{\sum_{j=1}^n \pi_j G(\widetilde r\mid M_j)}.
\end{equation*}

The user's ideal hypothesis is
\((s^{(i)}+\beta)\hat{\bm b}+\mathbf z\). Hence, the expected loss from the deviation
\(r\hat{\bm b}+\mathbf z'\) is
\begin{equation*}
\mathbb E_{\widetilde r\mid r}
\left[
\left\|
\Phi(\widetilde r;\mathbf M)\hat{\bm b}+\mathbf z'
-
\bigl((s^{(i)}+\beta)\hat{\bm b}+\mathbf z\bigr)
\right\|_2^2
\right].
\end{equation*}
Because \(\mathbf z'-\mathbf z\) is orthogonal to \(\hat{\bm b}\), this separates as
\begin{equation*}
\mathbb E_{\widetilde r\mid r}
\left[
\bigl(\Phi(\widetilde r;\mathbf M)-s^{(i)}-\beta\bigr)^2
\right]
+
\|\mathbf z'-\mathbf z\|_2^2.
\end{equation*}
For any scalar report \(r\), the second term is minimized by choosing
\(\mathbf z'=\mathbf z\). Thus, the user cannot improve their utility by changing the agreement component.

After fixing \(\mathbf z'=\mathbf z\), the deviation problem is exactly the one-dimensional noisy
problem for scalar parameter \(s^{(i)}\). Since \(\mathbf M=(M_1,\ldots,M_n)\) is a one-dimensional noisy
equilibrium, parameter \(s^{(i)}\) cannot profitably deviate from \(M_i\) to any other scalar report.
Therefore, no multidimensional parameter has a profitable deviation.

The learner's response is optimal because, under quadratic learner utility, the posterior mean is the
learner's best response. Hence, the lifted strategy profile is a multidimensional noisy equilibrium.

Finally, within the class of lifted equilibria that reveal \(\mathbf z\) exactly, the agreement
component creates no additional strategic loss. The only remaining optimization is the scalar noisy
problem in \(s\). Therefore, if \(\mathbf M\) is user-optimal by
Theorem~\ref{thm:noisy-optimal-report}, its lift is user-optimal within this lifted class.
\end{proof}

\textbf{Algorithm.}
The algorithm reuses Algorithm~\ref{alg:noisy-user-optimal-noisy-1d}. Given a private dataset \(D\), compute
\(\bm w_\ell^*(D)\). Decompose it into the conflict coordinate
\(s=\hat{\bm b}^{\top}\bm w_\ell^*(D)\) and the agreement component
\(\mathbf z=\bm w_\ell^*(D)-s\hat{\bm b}\). Run Algorithm~\ref{alg:noisy-user-optimal-noisy-1d}
on the finite scalar prior \(\pi_s\) with bias \(\beta=\|\bm b\|_2\) and noise level
\(\sigma\). Let the resulting user-optimal scalar tuple be
\(M^{s,\star}=(M_1^\star,\ldots,M_n^\star)\). If
\(s=s^{(i)}\), return the multidimensional report
$M^\star(D)=M_i^\star\hat{\bm b}+\mathbf z.$ 
The only additional multidimensional cost compared to Algorithm~\ref{alg:noisy-user-optimal-noisy-1d} is computing the projection
\(s=\hat{\bm b}^{\top}\bm w_\ell^*(D)\) and the agreement component
\(\mathbf z=\bm w_\ell^*(D)-s\hat{\bm b}\), which takes \(O(d)\) time. Therefore, using the complexity notations from Algorithm~\ref{alg:noisy-user-optimal-noisy-1d}, the overall complexity of the multidimensional algorithm is $
O\!\left(n^2T_{\mathrm{solve}}(n)+CnT_{\mathrm{BR}} + d\right)$.
Finally, the equilibrium certificate follows using th global
condition~\eqref{eq:noisy-global-IC} from Theorem~\ref{thm:user-optimal-noisy} in
Proposition~\ref{prop:eps-equilibrium}. In particular, for the returned
tuple \(M^\star\), define
\[
\varepsilon_{\mathrm{BR}}
=
\max_{i\in[n]}
\left[
\max_{r\in[0,1]} V_i(r;M^\star)
-
V_i(M_i^\star;M^\star)
\right].
\]
Then \(\varepsilon_{\mathrm{BR}}=0\) certifies exact best-response optimality,
while \(\varepsilon_{\mathrm{BR}}\le \varepsilon\) certifies an
\(\varepsilon\)-equilibrium.

\section{Missing Details of Section~\ref{sec:experiment}}
\label{sec:additional-experiment-setup}

\subsection{Experiment Setup}
\label{sec:appendix-eperiment-setup}

\paragraph{Datasets and preprocessing.}
We use four real-world datasets: Credit Approval~\cite{credit_approval},
School Admission~\cite{school_admission}, Census~\cite{adult_2}, and Prosper
Loan~\cite{prosper_loan}\footnote{\url{https://www.prosper.com/}}.
Table~\ref{tab:datasets} summarizes the dataset sizes, feature dimensions, and
prediction tasks. Numerical features are imputed with the training median and
standardized; categorical features are imputed with the training mode and one-hot
encoded, giving processed features (Proc. Feat.)

\begin{table}[ht]
\centering
\caption{Summary of datasets used in the experiments.}
\label{tab:datasets}
\begin{tabular}{lrrrrl}
\hline
\textbf{Dataset} & \textbf{\# Samples} & \textbf{\# Raw Feat.} &
\textbf{\# Proc. Feat.} & \textbf{\# Labels} & \textbf{Prediction Task} \\
\hline
Credit Approval & 690 & 15 & 46 & 2 & Approval \\
School Admission & 22{,}407 & 30 & 54 & 2 & Pass Exam \\
Census & 32{,}561 & 14 & 105 & 2 & Income \(>50{,}000\) \\
Prosper Loan & 84{,}853 & 64 & 189 & 11 & Risk Score \\
\hline
\end{tabular}
\end{table}

\paragraph{Learners.}
We evaluate two learner classes. The first is a linear SVM implemented with
\texttt{LinearSVC}; here \(\bm w_\ell^*(D)\) is the learned linear parameter
vector, including the intercept. The second is a two-layer neural network with
architecture \(\mathrm{Linear}(d,32)\)-ReLU-\(\mathrm{Linear}(32,C)\); here
\(\bm w_\ell^*(D)\) is the flattened vector of all learned parameters.
The SVM uses default hyperparameters with a fixed random seed. The neural network
is trained with Adam, learning rate \(10^{-3}\), batch size \(64\), and
cross-entropy loss. We use a 10\% validation split and early stopping after 10
epochs without validation-loss improvement.

\begin{table}[ht]
\centering
\caption{Model parameter dimensions after preprocessing.}
\label{tab:model-param-dims}
\begin{tabular}{lrr}
\hline
\textbf{Dataset} & \textbf{Linear SVM} & \textbf{MLP} \\
\hline
Credit Approval & 47 & 1570 \\
School Admission & 52--57 & 1762--1826 \\
Census & 104--106 & 3426--3458 \\
Prosper Loan & 2090 & 6443 \\
\hline
\end{tabular}
\end{table}

The ranges in Table~\ref{tab:model-param-dims} arise because preprocessing is fit
separately on each train split, so the set of one-hot encoded categories can vary
slightly across splits.

\paragraph{Empirical priors.}
A prior over private datasets induces a prior \(\pi_{\bm w}\) over
learner-optimal hypotheses \(\bm w_\ell^*(D)\). We approximate this prior by
stratified bootstrap resampling: for each bootstrap dataset, we retrain the
learner and store the resulting parameter vector. Unless otherwise stated, all
experiments use support size \(50\), i.e., \(50\) bootstrap datasets and
\(50\) learned models.

\paragraph{Noise and Bias.}
We evaluate both noiseless and noisy multidimensional algorithms. The bias vector
\(\bm b\in\mathbb R^d\) can be any nonzero vector for the algorithms. In each run, we sample a random conflict direction \(\bm g\sim\mathcal N(0,I_d)\) and set
$\bm b = c\frac{\bm g}{\|\bm g\|_2},$
so that \(\|\bm b\|_2=c\). In the noiseless \textbf{bias sweep}, we vary
$c\in\{0.005,0.02,0.05,0.10\}$. In the \textbf{noisy sweep}, we fix \(c=0.05\) and vary the Gaussian report-noise level over $\sigma\in\{0.005,0.02,0.05,0.10\}.$ This isolates the effect of communication noise from the effect of conflict magnitude. For each bias or noise level, we run three independent train/test splits with train splits treated as three private datasets. Therefore, we have a total of $4 \times 3 =12$ \textbf{trials}. All reported times are wall-clock times over three runs of the
corresponding sweep.

\paragraph{Strategic and non-strategic benchmarks.}
For each private dataset \(D\), the \emph{strategic} benchmark is the user-side
strategy computed by our algorithm. This strategy selects a report that maximizes
the user's expected utility given the learner's equilibrium interpretation of
reports. The \emph{non-strategic} benchmark is the uninfluential
baseline based on Section~\ref{sec:possibility-of-influence}. In this baseline,
the learner ignores the report and chooses a model using only its prior over
private datasets. Formally, the learner chooses
\[
\bm w_{\mathrm{uninf}}
\in
\arg\max_{\bm w\in\mathcal W}
\mathbb E_{D'\sim \pi}\!\left[U^\ell(D',h_{\bm w})\right].
\]
Thus, the user's report does not affect the learned model in the non-strategic
benchmark. The reported \textbf{utility gain} is
$10^3\cdot
\left(
U^r(D,h_{\bm w_{\mathrm{strat}}})
-
U^r(D,h_{\bm w_{\mathrm{uninf}}})
\right),$
where \(\bm w_{\mathrm{strat}}\) is the model induced by the strategic report.
The \textit{win rate} is the fraction of trials in which this quantity is strictly
positive.

\paragraph{Solver and implementation details.}
For the solver, we use \texttt{scipy.optimize.least\_squares}. For each candidate report pattern, the nonlinear first-order conditions are solved with box constraints restricting reports to \([0,1]\). A candidate is accepted only if the maximum unilateral best-response gain is at most \(10^{-4}\). We use three random restarts per pattern and keep the accepted candidate with the largest expected user utility. Posterior expectations under Gaussian noise are approximated using
Gauss--Hermite quadrature with 41 nodes.

All experiments were run on a MacBook Pro with an Apple M1 Pro processor,
16 GB of memory, and CPU-only execution. The implementation used Python 3.11.4,
NumPy 1.24.3, Pandas 1.5.3, SciPy 1.10.1, scikit-learn 1.3.0, and PyTorch 2.0.1.

\subsection{Additional Experiment Results}
\label{sec:additional-experiment-results}
\paragraph{Scalability of Algorithms.}
Tables~\ref{tab:noisy-runtime-breakdown} and~\ref{tab:noiseless-runtime-breakdown}
report end-to-end runtimes for the noisy and noiseless algorithms. The noiseless
dynamic program is negligible relative to model training across all datasets and
model classes. In contrast, the noisy solver dominates runtime for small and
medium datasets, while model training is also substantial on Prosper Loan. This is
expected because the noisy algorithm solves a nonlinear equilibrium problem,
whereas the noiseless algorithm reduces to an exact dynamic program. We use an
off-the-shelf \texttt{scipy.optimize.least\_squares} solver with a basic
configuration; optimized solvers, warm starts, specialized tuning, and
parallelization could further reduce solver runtime and improve overall
algorithmic efficiency.

\begin{table}[H]
\small
\centering
\caption{Scalability breakdown of noisy multi-dimensional algorithm  across noise sweep}
\label{tab:noisy-runtime-breakdown}
\begin{tabular}{llrrrr}
\hline
\textbf{Model} & \textbf{Dataset} & \(\dim(\bm w_\ell^*)\)
& \textbf{Train (s)} & \textbf{Solver (s)} & \textbf{Alg. Time (s)} \\
\hline
\multirow{4}{*}{Linear SVM}
& Credit Approval  & 47       & $0.00 \pm 0.00$ & $506.43 \pm 187.93$ & $506.44 \pm 187.93$ \\
& Census           & 104--106 & $0.10 \pm 0.01$ & $550.69 \pm 160.49$ & $550.80 \pm 160.49$ \\
& School Admission & 52--55   & $0.05 \pm 0.00$ & $454.93 \pm 230.68$ & $454.98 \pm 230.68$ \\
& Prosper Loan     & 2090     & $118.22 \pm 7.28$ & $668.29 \pm 179.56$ & $786.51 \pm 182.28$ \\
\hline
\multirow{4}{*}{MLP}
& Credit Approval  & 1570      & $0.86 \pm 0.09$ & $571.77 \pm 165.78$ & $572.63 \pm 165.80$ \\
& Census           & 3458      & $6.62 \pm 1.52$ & $520.61 \pm 171.97$ & $527.23 \pm 171.31$ \\
& School Admission & 1762--1826 & $3.46 \pm 0.90$ & $507.63 \pm 130.15$ & $511.08 \pm 130.56$ \\
& Prosper Loan     & 6443      & $49.30 \pm 6.47$ & $494.18 \pm 197.61$ & $543.48 \pm 198.56$ \\
\hline
\end{tabular}
\end{table}

\begin{table}[H]
\small
\centering
\caption{Scalability breakdown of noiseless multi-dimensional algorithm across bias sweep.}
\label{tab:noiseless-runtime-breakdown}
\begin{tabular}{llrrrr}
\hline
\textbf{Model} & \textbf{Dataset} & \(\dim(\bm w_\ell^*)\)
& \textbf{Train (s)} & \textbf{Exact DP (s)} & \textbf{Alg. Time (s)} \\
\hline
\multirow{4}{*}{Linear SVM}
& Credit Approval  & 47       & $0.00 \pm 0.01$ & $0.008 \pm 0.001$ & $0.01 \pm 0.01$ \\
& Census           & 104--106 & $0.10 \pm 0.01$ & $0.008 \pm 0.001$ & $0.11 \pm 0.01$ \\
& School Admission & 52--55   & $0.06 \pm 0.01$ & $0.008 \pm 0.000$ & $0.06 \pm 0.01$ \\
& Prosper Loan     & 2090     & $110.41 \pm 11.54$ & $0.009 \pm 0.001$ & $110.41 \pm 11.54$ \\
\hline
\multirow{4}{*}{MLP}
& Credit Approval  & 1570      & $0.42 \pm 0.39$ & $0.008 \pm 0.001$ & $0.43 \pm 0.39$ \\
& Census           & 3458      & $5.61 \pm 0.83$ & $0.008 \pm 0.001$ & $5.62 \pm 0.83$ \\
& School Admission & 1762--1826 & $3.27 \pm 1.55$ & $0.009 \pm 0.001$ & $3.27 \pm 1.55$ \\
& Prosper Loan     & 6443      & $52.26 \pm 6.12$ & $0.009 \pm 0.001$ & $52.27 \pm 6.12$ \\
\hline
\end{tabular}
\end{table}

\paragraph{Scaling with prior support size.}
Tables~\ref{tab:noisy-support-time} and~\ref{tab:noiseless-support-time} show how
runtime changes as the empirical-prior support size increases. The noiseless runtime
is essentially unchanged because the exact dynamic program is negligible compared
with training. The noisy runtime increases with support size, reflecting the growing
cost of solving and verifying the nonlinear scalar equilibrium problem. Solver time is
not monotone in dataset size because the noisy algorithm operates after compressing
the empirical prior onto the conflict direction.

\begin{table}[H]
\small
\centering
\caption{Noisy multidimensional algorithm time by empirical support size across noise sweep}
\label{tab:noisy-support-time}
\begin{tabular}{llrrrr}
\hline
\textbf{Model} & \textbf{Dataset} & \textbf{5} & \textbf{10} & \textbf{20} & \textbf{50} \\
\hline
\multirow{4}{*}{Linear SVM}
& Credit Approval & $1.49 \pm 0.37$ & $13.57 \pm 0.89$ & $235.42 \pm 44.83$ & $460.13 \pm 10.29$ \\
& School Admission & $1.87 \pm 0.21$ & $12.82 \pm 2.12$ & $195.35 \pm 28.68$ & $309.95 \pm 27.85$ \\
& Census & $1.71 \pm 0.40$ & $13.41 \pm 0.81$ & $213.78 \pm 16.40$ & $524.68 \pm 66.30$ \\
& Prosper Loan & $100.88 \pm 10.04$ & $111.02 \pm 9.89$ & $228.71 \pm 18.56$ & $791.56 \pm 26.93$ \\
\hline
\multirow{4}{*}{MLP}
& Credit Approval & $2.34 \pm 0.28$ & $12.54 \pm 1.39$ & $168.23 \pm 29.34$ & $544.22 \pm 60.31$ \\
& School Admission & $4.25 \pm 0.79$ & $13.35 \pm 0.52$ & $172.01 \pm 21.53$ & $492.44 \pm 47.01$ \\
& Census & $5.92 \pm 2.09$ & $17.55 \pm 3.62$ & $184.45 \pm 21.16$ & $512.21 \pm 12.54$ \\
& Prosper Loan & $37.65 \pm 11.62$ & $47.71 \pm 10.10$ & $168.75 \pm 10.21$ & $527.27 \pm 35.50$ \\
\hline
\end{tabular}
\end{table}

\begin{table}[H]
\small
\centering
\caption{Noiseless multidimensional algorithm time by empirical support size across bias sweep}
\label{tab:noiseless-support-time}
\begin{tabular}{llrrrr}
\hline
\textbf{Model} & \textbf{Dataset} & \textbf{5} & \textbf{10} & \textbf{20} & \textbf{50} \\
\hline
\multirow{4}{*}{Linear SVM}
& Credit Approval & $0.00 \pm 0.00$ & $0.00 \pm 0.00$ & $0.00 \pm 0.00$ & $0.01 \pm 0.00$ \\
& School Admission & $0.06 \pm 0.01$ & $0.06 \pm 0.01$ & $0.06 \pm 0.01$ & $0.07 \pm 0.01$ \\
& Census & $0.09 \pm 0.00$ & $0.09 \pm 0.00$ & $0.09 \pm 0.00$ & $0.11 \pm 0.00$ \\
& Prosper Loan & $107.43 \pm 12.56$ & $107.43 \pm 12.56$ & $107.43 \pm 12.56$ & $110.43 \pm 12.44$ \\
\hline
\multirow{4}{*}{MLP}
& Credit Approval & $0.49 \pm 0.38$ & $0.49 \pm 0.38$ & $0.49 \pm 0.38$ & $0.41 \pm 0.41$ \\
& School Admission & $2.16 \pm 0.42$ & $2.16 \pm 0.42$ & $2.16 \pm 0.42$ & $3.24 \pm 1.72$ \\
& Census & $4.10 \pm 1.58$ & $4.10 \pm 1.58$ & $4.10 \pm 1.58$ & $6.37 \pm 1.08$ \\
& Prosper Loan & $32.21 \pm 9.36$ & $32.21 \pm 9.36$ & $32.21 \pm 9.36$ & $60.63 \pm 5.23$ \\
\hline
\end{tabular}
\end{table}
\paragraph{Approximate Equilibria under the Numerical Solver.}
Table~\ref{tab:noisy-solver-diagnostics-sweep} reports diagnostics for the
numerical solver used in Algorithm~\ref{alg:noisy-user-optimal-noisy-1d}.
The \emph{success rate} is the fraction of runs in which the solver finds at
least one candidate report profile that passes our equilibrium checks. It is
\(1.00\) in all settings, so every run produced at least one verified candidate.
The \emph{number of verified candidates} is the average number of candidate
equilibria that passed these checks.

The \emph{residual} measures how accurately a candidate satisfies the
first-order equilibrium equations used to construct it,
namely~\eqref{eq:noisy-z-interior}--\eqref{eq:noisy-z-upper} from
Proposition~\ref{prop:noisy-equilibrium-structure}. A residual of zero means
that these equations are satisfied up to numerical precision. The residual is
zero in all settings except MLP Prosper Loan, where it is
\(3.03\times 10^{-3}\).

The \emph{BR violation} measures the largest utility improvement available to
the user from deviating to any alternative report \(r\in[0,1]\). This checks the condition~\eqref{eq:noisy-global-IC} from
Theorem~\ref{thm:user-optimal-noisy}. Smaller values indicate that the candidate
is closer to a true equilibrium. Across all settings, the maximum BR violation is
\(4.98\times 10^{-6}\), so no possible parameter can gain more than this amount by a
continuous unilateral deviation. Therefore the selected candidate, which
maximizes~\eqref{eq:noisy-user-optimal-program} among verified candidates, as an
approximately user-optimal noisy equilibrium.

\begin{table}[H]
\centering
\small
\setlength{\tabcolsep}{3pt}
\caption{Diagnostics for the noisy multi-dimensional algorithm across the noise sweep.}
\label{tab:noisy-solver-diagnostics-sweep}
\begin{tabular}{llcccc}
\hline
\textbf{Model} & \textbf{Dataset} & \textbf{Success Rate} & \textbf{Verified Candidates} & \textbf{Residuals} & \textbf{BR Violation} \\
\hline
\multirow{4}{*}{Linear SVM}
& Credit Approval & 1.00 & $3.88 \pm 1.25$ & $0.00$ & $1.70{\times}10^{-6}$ \\
& School Admission & 1.00 & $2.00 \pm 0.00$ & $0.00$ & $0.00$ \\
& Census  & 1.00 & $4.25 \pm 1.67$ & $0.00$ & $2.06{\times}10^{-6}$ \\
& Prosper Loan& 1.00 & $2.75 \pm 0.46$ & $0.00$ & $1.48{\times}10^{-6}$ \\
\hline
\multirow{4}{*}{MLP}
& Credit Approval  & 1.00 & $3.50 \pm 1.31$ & $0.00$ & $1.83{\times}10^{-6}$ \\
& School Admission  & 1.00 & $5.38 \pm 2.13$ & $0.00$ & $1.52{\times}10^{-6}$ \\
& Census  & 1.00 & $4.62 \pm 2.83$ & $0.00$ & $1.17{\times}10^{-6}$ \\
& Prosper Loan& 1.00 & $6.12 \pm 4.61$ & $3.03{\times}10^{-3}$ & $\mathbf{4.98{\times}10^{-6}}$ \\
\hline
\end{tabular}
\end{table}

\paragraph{Approximate Equilibria under Violated Factorization.}
Table~\ref{tab:epsilon-checks} evaluates the certificate from
Proposition~\ref{prop:eps-equilibrium} on the bias sweep. We use the noiseless
setting so that any failure of the exact equilibrium certificate comes from
violations of the factorization condition, rather than from numerical integration
or noisy-solver error. The quantity \(\varepsilon_{\mathrm{IC}}\) measures the
largest utility gain that any possible parameter can obtain by deviating from its assigned
partition/report. Thus, \(\varepsilon_{\mathrm{IC}}=0\) means that the fixed-point
condition~\eqref{eq:multidim-fixed-point} is satisfied exactly and the returned
partition is certified as an exact equilibrium. Positive values certify an
\(\varepsilon_{\mathrm{IC}}\)-equilibrium in the sense of
Proposition~\ref{prop:eps-equilibrium}.

The pass rate is the fraction of trials with \(\varepsilon_{\mathrm{IC}}=0\).
Prosper Loan with MLP achieves pass rate \(1.00\) and
\(\varepsilon_{\mathrm{IC}}=0\), confirming that the returned partition is an
exact equilibrium according to the certificate. MLP Credit Approval also performs
well, with pass rate \(0.92\) and mean \(\varepsilon_{\mathrm{IC}}=0.014\).
School Admission with Linear SVM shows small violations, with mean \(0.028\) and
maximum \(0.122\), yielding tight \(\varepsilon\)-equilibria. The largest
violations occur on Census. Linear SVM has pass rate \(0.08\) and maximum
\(\varepsilon_{\mathrm{IC}}=4.792\), while MLP has pass rate \(0.50\) and
maximum \(\varepsilon_{\mathrm{IC}}=28.737\). These larger values indicate
meaningful factorization violations, so the exact guarantee of
Theorem~\ref{thm:factorized-lift} does not apply in those settings.

\begin{table}[H]
\centering
\small
\caption{Equilibrium verification using \(\varepsilon_{\mathrm{IC}}\) checks for bias sweep.}
\label{tab:epsilon-checks}
\begin{tabular}{llrrrr}
\hline
\textbf{Model} & \textbf{Dataset} & \textbf{Trials} & \textbf{Pass Rate}
& \(\boldsymbol{\varepsilon_{\mathrm{IC}}}\) \textbf{Mean}
& \(\boldsymbol{\varepsilon_{\mathrm{IC}}}\) \textbf{Max} \\
\hline
\multirow{4}{*}{Linear SVM}
& Credit Approval   & 12 & 0.17 & 0.890 & 1.352 \\
& School Admission  & 12 & 0.25 & 0.028 & 0.122 \\
& Census            & 12 & 0.08 & 1.609 & 4.792 \\
& Prosper Loan      & 12 & 0.58 & 0.151 & 1.008 \\
\hline
\multirow{4}{*}{MLP}
& Credit Approval   & 12 & 0.92 & 0.014 & 0.172 \\
& School Admission  & 12 & 0.67 & 0.923 & 5.239 \\
& Census            & 12 & 0.50 & 5.588 & 28.737 \\
& Prosper Loan      & 12 & 1.00 & 0.000 & 0.000 \\
\hline
\end{tabular}
\end{table}